\PassOptionsToPackage{hyperfootnotes=false}{hyperref}
\RequirePackage{iftex}
\ifPDFTeX
  \errmessage{This manuscript requires XeLaTeX or LuaLaTeX with OpenType text fonts}
\fi
\documentclass[a4paper,fleqn]{cas-sc}

\usepackage[no-math]{fontspec}
\usepackage[expansion=false]{microtype}
\usepackage{graphicx}
\usepackage{amssymb}
\usepackage{amsmath}
\usepackage{amsthm}
\usepackage{pifont}
\usepackage{algorithm}
\usepackage{algpseudocode}
\usepackage{placeins}
\usepackage[table]{xcolor}

\definecolor{textmain}{HTML}{202124}
\usepackage{booktabs}
\usepackage{multirow}
\usepackage{makecell}
\usepackage{threeparttable}
\usepackage{etoolbox}

\usepackage{tabularx}
\usepackage{array}
\usepackage{ragged2e}
\usepackage{siunitx}
\robustify\bfseries
\usepackage[numbers,sort&compress,square]{natbib}
\setcitestyle{numbers,square,comma}

\usepackage{xurl}
\usepackage{hyperref}
\hypersetup{
unicode=true,
pdfencoding=auto,
colorlinks=true,
hypertexnames=false,
linkcolor=textmain,
filecolor=textmain,
urlcolor=textmain,
citecolor=textmain,
breaklinks=true,
pdftitle={Not All Agreement Counts as Corroboration: Provenance-Conserving Multi-View Fusion for Typed Action Admission in Human--Robot Collaboration},
pdfauthor={Jin et al.},
pdfsubject={Preprint},
pdfkeywords={low-quality multi-view fusion, provenance-aware fusion, evidence countability, multimodal foundation models, selective action admission, human--robot collaboration}
}

\usepackage[nameinlink]{cleveref}

\Crefname{figure}{Fig.}{Figs.}
\crefname{figure}{Fig.}{Figs.}
\Crefname{equation}{Eq.}{Eqs.}
\crefname{equation}{Eq.}{Eqs.}
\Crefname{table}{Table}{Tables}
\crefname{table}{Table}{Tables}
\Crefname{section}{Section}{Sections}
\crefname{section}{Section}{Sections}
\Crefname{algorithm}{Algorithm}{Algorithms}
\crefname{algorithm}{Algorithm}{Algorithms}

\usepackage{needspace}

\newcolumntype{Q}[1]{>{\RaggedRight\hyphenpenalty=10000\exhyphenpenalty=10000\arraybackslash}p{#1}}
\newcolumntype{M}[1]{>{\centering\arraybackslash}p{#1}}
\newcolumntype{Y}{>{\RaggedRight\arraybackslash}X}
\newcommand{\IFTableSetup}{%
\centering
\small
\rmfamily
\renewcommand{\arraystretch}{1.08}%
\setlength{\tabcolsep}{3.5pt}%
}
\newlength{\PACTPrimaryTableWd}
\newcommand{\cmark}{\ding{51}}
\newcommand{\xmark}{\ding{55}}

\newtheorem{theorem}{Theorem}
\newtheorem{proposition}{Proposition}
\newtheorem{corollary}{Corollary}
\newtheorem{assumption}{Assumption}
\newtheorem{definition}{Definition}
\Crefname{proposition}{Proposition}{Propositions}
\crefname{proposition}{Proposition}{Propositions}
\Crefname{definition}{Definition}{Definitions}
\crefname{definition}{Definition}{Definitions}
\Crefname{assumption}{Assumption}{Assumptions}
\crefname{assumption}{Assumption}{Assumptions}
\Crefname{corollary}{Corollary}{Corollaries}
\crefname{corollary}{Corollary}{Corollaries}
\Crefname{theorem}{Theorem}{Theorems}
\crefname{theorem}{Theorem}{Theorems}

\makeatletter
\patchcmd{\section}{15pt}{12pt}{}{\PackageWarning{pact-layout}{Section spacing patch failed; class spacing retained}}
\patchcmd{\subsection}{10pt}{7pt}{}{\PackageWarning{pact-layout}{Subsection spacing patch failed; class spacing retained}}
\patchcmd{\subsubsection}{10pt}{7pt}{}{\PackageWarning{pact-layout}{Subsubsection spacing patch failed; class spacing retained}}
\patchcmd{\paragraph}{10pt}{7pt}{}{\PackageWarning{pact-layout}{Paragraph spacing patch failed; class spacing retained}}
\makeatother

\begin{document}
\let\printorcid\relax
\let\WriteBookmarks\relax
\renewcommand{\floatpagefraction}{0.85}
\renewcommand{\textfraction}{0.05}
\renewcommand{\topfraction}{0.95}
\renewcommand{\bottomfraction}{0.85}
\setlength{\textfloatsep}{7pt plus 2pt minus 2pt}
\setlength{\floatsep}{6pt plus 2pt minus 2pt}
\setlength{\intextsep}{6pt plus 2pt minus 2pt}
\setlength{\abovecaptionskip}{3pt}
\setlength{\belowcaptionskip}{1pt}
\ExplSyntaxOn
\cs_set:Npn \__make_tbl_caption:nn #1#2
{
  \l_tbl_align_tl
  \skip_vertical:N \l_tbl_abovecap_skip
  \parbox{ \l_tbl_width_dim }
  {\RaggedRight\rmfamily\small\textbf{\color{scolor}#1.}~#2\par\vskip4pt }
  \skip_vertical:N \l_tbl_belowcap_skip
}
\cs_set:Npn \__make_fig_caption:nn #1#2
{
  \l_fig_align_tl
  \skip_vertical:N \l_fig_abovecap_skip
  \parbox{ \l_fig_width_dim }
  {\RaggedRight\rmfamily\small\textbf{\color{scolor}#1:}~#2\par }
  \skip_vertical:N \l_fig_belowcap_skip
}
\cs_set:Npn \__reset_fig:
{
  \tl_set:Nx \l_fig_pos_tl { t }
  \tl_set:Nx \l_fig_cols_tl { 1 }
  \tl_set:Nn \l_fig_align_tl { \raggedright }
  \skip_set:Nn \l_fig_abovecap_skip { 6pt }
  \skip_set:Nn \l_fig_belowcap_skip { 6pt }
  \skip_set:Nn \l_fig_abovefig_skip { 6pt }
  \skip_set:Nn \l_fig_belowfig_skip { 6pt }
}
\ExplSyntaxOff
\makeatletter
\def\fps@figure{htbp}
\def\fps@table{htbp}
\makeatother

\shorttitle{PACT}
\shortauthors{Jin et al.}

\makeatletter
\@ifundefined{frontmatter}{\newenvironment{frontmatter}{}{\relax}}{}
\makeatother
\begin{frontmatter}

% ===== Title =====
\title[mode=title]{\texorpdfstring{{\fontsize{15}{17}\selectfont\textls[-35]{Not All Agreement Counts as Corroboration: Provenance-Conserving\\
Multi-View Fusion for Typed Action Admission in Human--Robot Collaboration}}}{Not All Agreement Counts as Corroboration: Provenance-Conserving Multi-View Fusion for Typed Action Admission in Human--Robot Collaboration}}
\tnotemark[1]

% ===== Authors =====
\author[1]{Zekai Jin}
\ead{zekai.jin@mail.mcgill.ca}
\author[2]{Hanrong Zhang}
% \ead{hzhan135@uic.edu}
\author[1,3]{Yihong Tang}
% \ead{yihong.tang@mail.mcgill.ca}
\author[4]{Fei Hu}
% \ead{fei.hu@ntnu.no}
\author[5,6]{Zhen Dong}
% \ead{zhendong@berkeley.edu}
\author[1,7]{Yi Shao}
\cormark[1]
\ead{yi.shao@ubc.ca}
\tnotetext[1]{Project page \href{https://github.com/ZekaiJ/PACT}{github.com/ZekaiJ/PACT}}

% ===== Addresses =====
\address[1]{Department of Civil Engineering, McGill University, 817 Sherbrooke Street West, Montreal, Quebec, Canada}
\address[2]{Department of Computer Science, University of Illinois Chicago, 850 West Taylor Street, Chicago, Illinois 60607, United States}
\address[3]{Mila -- Quebec AI Institute, 6666 Saint-Urbain Street, Montreal, Quebec, Canada}
\address[4]{Department of Structural Engineering, Norwegian University of Science and Technology, Richard Birkelands vei 1A, Trondheim, Norway}
\address[5]{Department of Computer Science, University of California, Santa Barbara, 2104 Harold Frank Hall, Santa Barbara, California, United States}
\address[6]{NVIDIA Corporation, 2788 San Tomas Expressway, Santa Clara, California, United States}
\address[7]{Department of Civil Engineering, University of British Columbia, 6250 Applied Science Lane, Vancouver, British Columbia, Canada}
\cortext[1]{Corresponding author.}

% ===== Abstract =====
\begin{abstract}
Improved probability scores do not by themselves establish that evidence has been counted correctly. Repeated inference over a single acquired observation can improve predictions without introducing an additional evidential origin. Source-local numerical attributes alone are in general insufficient to distinguish repeated derivations from separately countable acquisitions. PACT (\textbf{P}rovenance-\textbf{A}ware evidence \textbf{C}onservation and \textbf{T}yped action admission) represents this distinction through a supplied provenance partition that separates the magnitude of evidence from its countability. Under singleton fidelity and insertion non-amplification, the coordinatewise meet is the unique pointwise greatest admissible within-component rule. Accumulation across components further relies on stated assumptions of commensurability and separate-component additivity. Matched reassignments vary the supplied counting relation while holding numerical outputs fixed. In four of the 12 replicated-source tests on HandWritten, false refinement relative to the reference grouping lowers the macro-averaged negative log-likelihood and the Brier score while raising the normalized common-support area under the risk--coverage curve (ncsAURC). In the controlled handover benchmark, once the constructed adversarial-consensus condition is removed, aggregation over the provenance partition lowers ncsAURC by 0.056 relative to singleton aggregation under the same score functional. The corroboration contrast then disappears, and the ranking of methods remains dependent on the selection score. The offline, reference-based human--robot collaboration study employs four prompts per camera. With all other admission inputs held fixed, duplicating each prompt output within its camera from multiplicity one to eight leaves all 720 typed responses of PACT per checkpoint unchanged. Posterior quality and evidence countability therefore require separate evaluation.
\end{abstract}

% ===== infoauthors =====
\gdef\infoauthors{Jin et al.}

% ===== maketitle =====
% Local tolerance for the CAS frontmatter email box.
\hfuzz=118pt
\maketitle
\hypersetup{pdfsubject={Preprint}}
\hfuzz=.2pt
\end{frontmatter}

% ===========================
% SECTION 1: Introduction
% ===========================
\section{Introduction}
\label{sec:introduction}

Repeated inference can increase agreement without the addition of an acquired observation. Re-prompting, resampling, and test-time candidate generation may improve prediction or search from the same input~\citep{Wang2023SelfConsistency,Kwok2025RoboMonkey,Zhao2026TapSampling}, and related models may also share correlated errors~\citep{Kim2025CorrelatedErrors}. Once agreement is interpreted as corroboration, the number of outputs alone becomes insufficient. Whether their support may be counted separately depends on how the outputs were acquired or derived. This relation is what we term \emph{evidence countability}.

When fused predictions determine whether an embodied system may proceed, an error in the counting of evidence can propagate into action admission. Human--robot collaboration (HRC) places robots in shared workspaces, where a single decision may combine commands, scene geometry, path risk, and the outputs of several vision--language models~\citep{Lasota2015HumanAware}. The decision layer selects an action from these outputs, accounts for which support may accumulate separately under the supplied relation, and checks that the remaining release conditions are satisfied. This pre-execution authorization is referred to as \emph{action admission}. A \emph{typed} rule maps the first unmet condition to \textsc{hold}, \textsc{confirm}, or \textsc{fallback}. Physical safety nevertheless continues to depend on controllers, execution monitoring, and safeguards that remain active during execution~\citep{Lasota2017SafeHRI}.

\begin{figure}[pos=t]
\includegraphics[width=\textwidth]{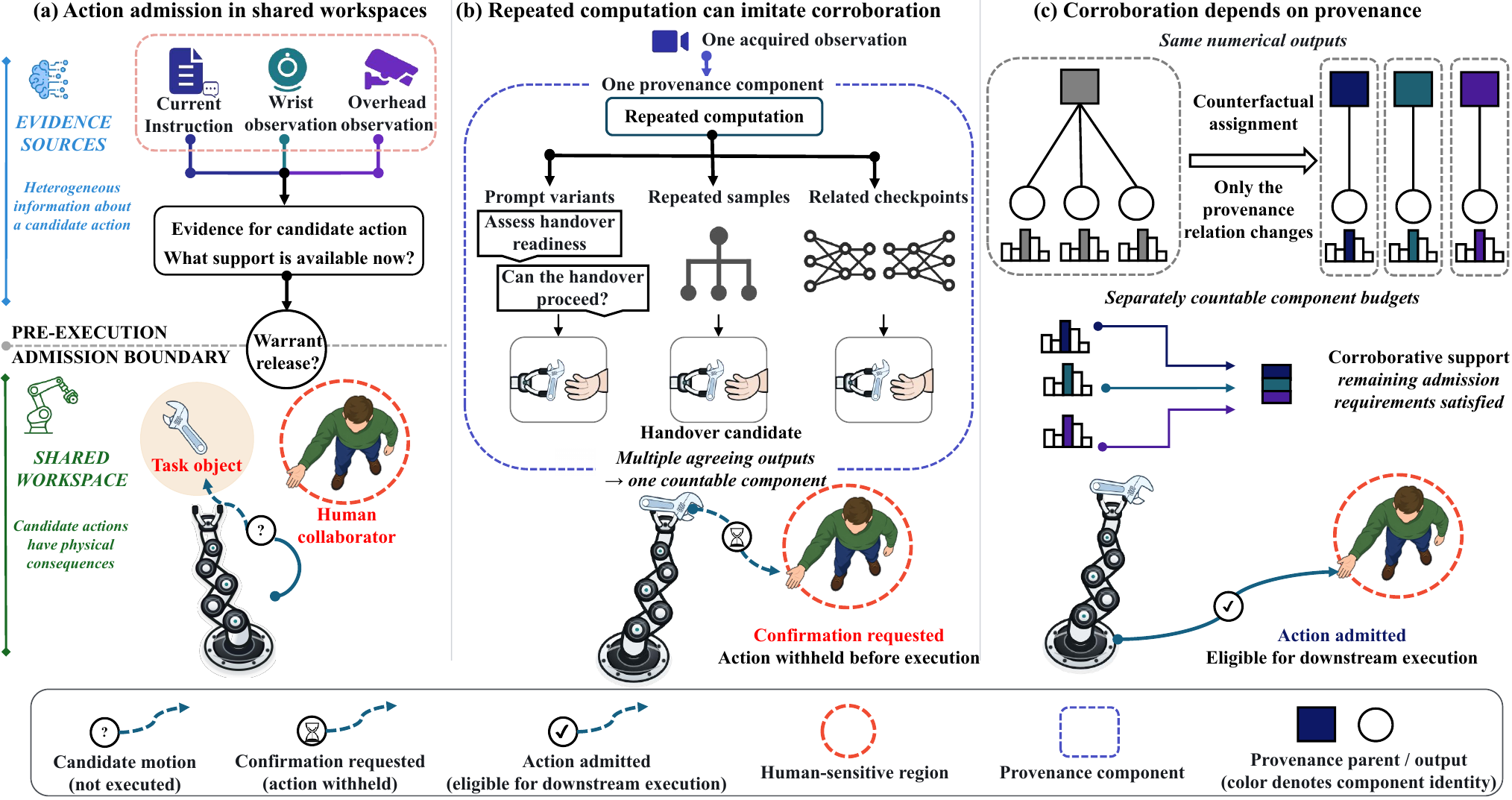}
\caption{Why agreement need not constitute corroboration. (a) Heterogeneous evidence informs whether a candidate handover may proceed to execution in a shared workspace. (b) Repeated inference over a single acquired observation can yield agreeing candidates without the addition of a further evidential origin. The admission policy illustrated here requests confirmation. (c) Reassigning the same numerical outputs to different parents changes which support may accumulate separately. Admission still depends on the remaining release conditions.}
\label{fig:pact_concept}
\end{figure}

Prior work has long recognized that relationships among sources can alter the semantics of fusion. J{\o}sang et al.\ distinguish cumulative fusion for separately acquired observations from averaging when sensors observe the same outcomes, and beliefs derived by bootstrapping from a single training set have likewise been treated as non-independent~\citep{Josang2010CumulativeAveraging,Francois2003ResampleCombine}. More generally, cautious, idempotent, and homogeneous fusion rules limit reinforcement whenever source distinctness or common information is uncertain~\citep{Denoeux2008Cautious,Dubois2016BasicPrinciples,Taylor2019Homogeneous}. PACT (\textbf{P}rovenance-\textbf{A}ware evidence \textbf{C}onservation and \textbf{T}yped action admission) isolates a different identifiability problem. Two identical predictions may arise either from repeated processing of one camera image or from separately acquired images. Even when the source-local numerical attributes are identical, the two cases may require different counting assignments under the supplied acquisition or derivation relation. No fusion map whose input is restricted to those local attributes can recover the distinction. \Cref{prop:opinion_only_indistinguishability} formalizes this value-only indistinguishability.

Evidence countability is encoded in PACT through a supplied provenance partition. Source-local adapters map outputs to nonnegative evidence magnitudes, whereas overlaps among the supplied parent sets induce the components that determine which outputs may accumulate separately.

Within each component, the coordinatewise minimum of the member evidence vectors is retained. Under the stated assumptions of commensurability and separate-component additivity, the retained budgets are then added across components that are supplied as separately countable.

Matched reassignments hold source-local evidence fixed, while exact-copy insertion tests conservation as the source set expands. After removing the constructed adversarial-consensus condition, provenance-partition aggregation lowers normalized common-support area under the risk--coverage curve (ncsAURC) by 0.056 relative to singleton aggregation under the common \(1-u\) score functional. The corroboration contrast disappears. Learned multi-view and multi-camera outputs reproduce the structural budget response, while predictive consequences depend on the model, dataset, and score.

The offline HRC study tests computational multiplicity under acquisition-based grouping. With four prompts per camera and all other admission inputs held fixed, exact within-camera duplication from multiplicity one to eight leaves all 720 typed responses of PACT per checkpoint unchanged. Under the target-identity and strict positive-change temporal requirements, camera-grouped PACT admits 47 of the 57 reference-consistent candidates produced by Qwen3-VL-32B. The retention ordering varies across the compared counting rules and across checkpoints.

Evidence countability is formalized in PACT as relational information that cannot in general be recovered from source-local numerical attributes alone. Singleton fidelity and insertion non-amplification characterize the coordinatewise meet as the unique pointwise greatest admissible within-component rule, and accumulation across components requires explicit assumptions of commensurability and additivity. Matched interventions on constructed and on learned outputs indicate that counting replicated outputs separately can increase the retained evidence even when the negative log-likelihood (NLL) and the Brier score improve. In some cases, the expected calibration error computed with ten equal-width bins (ECE$_{10}$) improves as well. The offline HRC study examines how acquisition-based counting affects typed admission.

% =======================
% SECTION 2: Related Work
% =======================
\section{Related Work}
\label{sec:related_work}

% ==========================================================
% SECTION 2.1: Low-quality multi-view fusion under uncertainty
% ==========================================================
\subsection{Low-quality multi-view fusion under uncertainty}
\label{subsec:rw_low_quality_fusion}

Evidence theory and evidential learning provide explicit representations of support and of uncertainty, whereas low-quality multimodal learning addresses noise, missingness, quality imbalance, and sample-dependent degradation~\citep{Shafer1976Evidence,Sensoy2018EDL,Zhang2024LowQualitySurvey}. Modern multi-view methods adapt the effective contribution of a view according to estimated uncertainty, reliability, conflict, or learned quality~\citep{Han2021TMC,Zhang2023QMF,Xu2024RCML,Liu2024CCML}. These mechanisms regulate the strength of a contribution through evidence representations, source discounting, or numerical weights.

A separate line of work models the information shared across views. Zhang et al.\ formalize multi-view common information through a G{\'a}cs--K{\"o}rner-inspired variable that can be recovered deterministically from every view, and they separate common representations from view-unique ones~\citep{Zhang2024CUMI}. Common information concerns shared representational content, while evidence countability concerns which outputs may contribute support separately. In PACT, source-local adapters specify nonnegative evidence magnitudes, whereas a supplied acquisition or derivation relation specifies the counting structure.

% =======================================================
% SECTION 2.2: Dependence-aware fusion and provenance structure
% =======================================================
\subsection{Dependence-aware fusion and provenance structure}
\label{subsec:rw_dependent_information}

The relationships among sources have long been understood to affect the semantics of fusion. Covariance intersection remains conservative when the cross-covariances are unknown, and data-incest methods exploit the structure of information flow in order to prevent recursive reuse~\citep{Julier1997CovarianceIntersection,Hamdi2013DataIncest}. Taylor and Bishop address unknown common information by means of degree-one homogeneous functionals, under which a common probabilistic factor enters only once. The normalized power-mean family that they propose includes linear and geometric pooling, with the pointwise minimum arising as a limiting case~\citep{Taylor2019Homogeneous}.

Work on belief functions makes the role of source relationships explicit. J{\o}sang et al.\ distinguish the cumulative fusion of observations collected in separate periods, averaging when sensors observe the same outcomes, and conjunctive fusion for orthogonal constraints. The worked hierarchy that they present averages simultaneous observations of the same aspect, combines an orthogonal aspect conjunctively, normalizes the resulting observation, and only then accumulates observations that are separate~\citep{Josang2010CumulativeAveraging}. Fran{\c{c}}ois et al.\ treat belief structures produced from bootstrap resamples of a single training set as non-independent and average them at the credal level. Such repeated computation may nevertheless improve the representation of uncertainty, the rejection of unreliable cases, and subsequent fusion~\citep{Francois2003ResampleCombine}. The cautious rule of Den{\oe}ux addresses bodies of evidence that are reliable but not distinct, and hierarchical fusion in the transferable belief model likewise varies the combination according to the structure of the sources~\citep{Denoeux2008Cautious,HaDuong2008Hierarchical}.

Dubois et al.\ formulate fusion through representation-independent principles that include minimal commitment, discuss idempotence as appropriate whenever source independence has not been established, and reject majority reinforcement as a universal requirement of fusion on the grounds that sources may be redundant~\citep{Dubois2016BasicPrinciples}. Idempotent, cautious, and hierarchical combination therefore rest on established precedents whose interpretation depends on the represented information and on the assumed relationships among sources.

The provenance semirings of Green et al.\ propagate input annotations through computational derivations~\citep{Green2007ProvenanceSemirings}. Kowalski and Martin apply provenance after belief combination in order to attribute a final decision to the contributing basic belief assignments of the sources~\citep{KowalskiMartin2018Provenance}.

In PACT, the counting units are made explicit inputs to fusion. Supplied acquisition or derivation records define these units before the numerical evidence is combined. Within each unit, singleton fidelity and insertion non-amplification characterize the coordinatewise meet as the pointwise greatest admissible rule on nonnegative evidence. Across components, the retained budgets are added under the commensurability and separate-component additivity assumptions stated in Section~\ref{subsec:pre_execution_interface}. The components do not certify statistical independence, and the meet does not estimate latent shared information.

% =======================================================
% SECTION 2.3: Selective prediction and action admission
% =======================================================
\subsection{Selective prediction and action admission}
\label{subsec:rw_selective_risk}

Selective prediction addresses the question of whether a prediction should be retained. Risk--coverage formulations and SelectiveNet formalize abstention~\citep{ElYaniv2010SelectiveClassification,Geifman2019SelectiveNet}, whereas Adaptive Rejection removes low-quality views prior to combination~\citep{Liu2025AdaptiveRejection}. Conformal prediction yields prediction sets with coverage guarantees that hold under its own assumptions~\citep{Angelopoulos2023Conformal}. In embodied decisions, KnowNo and Act or Ask employ uncertainty in order to request help or to defer a decision~\citep{Ren2023KnowNo,Huang2026ActOrAsk}.

Hierarchical fusion together with rejection also predates current embodied models. Suutala and R{\"o}ning combine classifier posteriors obtained from several feature representations of the same footstep, normalize the first-stage result, aggregate multiple footsteps at a second stage by means of fixed pooling rules, and then apply rejection thresholds selected on validation data~\citep{Suutala2008FloorReject}. That method distinguishes same-event fusion, cross-event fusion, and downstream rejection. In PACT, the counting partition is an explicit input that may vary independently of the numerical outputs, and a within-component budget rule governs how those outputs contribute. Typed admission addresses the conditions under which a candidate action is authorized, which is a different question from whether a prediction should be retained.

% =========================================================================
% SECTION 2.4: Embodied foundation models and repeated computation
% =========================================================================
\subsection{Embodied foundation models and repeated computation}
\label{subsec:rw_vla_candidate_evidence}

Embodied foundation models render computational multiplicity particularly visible. Vision--language--action (VLA) policies and structured planners map multimodal context to actions, skills, or control references~\citep{Kim2024OpenVLA,Ichter2023SayCan,Huang2023VoxPoser,Li2026GFVLA}. Test-time methods may then generate additional candidates from the same observation. RoboMonkey samples and perturbs VLA proposals before verification, whereas TapSampling evaluates sampled actions through predicted task progress~\citep{Kwok2025RoboMonkey,Zhao2026TapSampling}. Computation of this kind can improve prediction or candidate selection without constituting a further acquisition.

Safety-oriented work addresses complementary aspects of embodied decision making. SafeVLA constrains policy learning, EMBGuard evaluates action-conditioned hazards, and UNISafe combines epistemic uncertainty with a reachability-based safety filter~\citep{Zhang2025SafeVLA,Choi2026EMBGuard,Seo2025UNISafe}. The focus of PACT lies in how supplied countability enters pre-execution admission, while downstream control, execution monitoring, and physical safeguards remain distinct responsibilities~\citep{Lasota2017SafeHRI}.

% ================================================================
% SECTION 3: PACT Fusion and Typed Action Admission
% ================================================================
\section{PACT Fusion and Typed Action Admission}
\label{sec:methodology}

% ==========================
% SECTION 3.1: Method overview
% ==========================
\subsection{Method overview}
\label{subsec:method_overview}

An evidence magnitude is assigned to each source in PACT, and provenance is used to determine which outputs may be counted separately. Source adapters produce nonnegative evidence vectors, whose coordinatewise meet is retained within each provenance component, and the component budgets are then added under the assumptions stated below. The resulting evidence determines the posterior and, in the controlled benchmark, the selection score applied before typed admission (\Cref{fig:pact_concept}).

The term \emph{PACT fusion} denotes this provenance-conditioned evidence accounting together with posterior projection and scoring, whereas \emph{PACT} denotes the complete procedure by which a contract is selected and its admission is determined. Algorithm~\ref{alg:pact_decision_path} specifies the score-thresholded path of the controlled benchmark. \Cref{tab:notation} summarizes the notation, and Appendix~\ref{sec:proofs} provides the proofs and the perturbation bounds.

\begin{algorithm}[t]
\small
\caption{PACT fusion and typed action admission in the score-thresholded benchmark}
\label{alg:pact_decision_path}
\begin{algorithmic}[1]

\Require Decision instance $x$, expected sources $\mathcal S(x)=(\xi_i(x))_{i=1}^{J}$, evidence adapters $(\Phi_i)_{i=1}^{J}$, base-eligibility predicate $\zeta_{\mathrm{base}}$, threshold $\tau$, base rate $\mathbf a$, prior strength $W$, and typed admission rule $\mathcal V$
\Ensure Final decision $Y(x)$
\Statex \textbf{Construct evidence and countability}
\State Identify observed, numerically valid sources using \Cref{eq:observed_source_set}
\State Construct $\Pi_P(x)$ from parent-set overlaps using \Cref{eq:provenance_partition}
\For{each expected source $i=1,\ldots,J$}
\State Map $\xi_i(x)$ to nonnegative evidence $\mathbf e_i\gets\Phi_i(\xi_i(x))$
\EndFor
\Statex \textbf{Conserve and accumulate support}
\For{each component $C\in\Pi_P(x)$}
\State $\mathbf b_C\gets\bigwedge_{i\in C}\mathbf e_i$
\EndFor
\State $\mathbf E_{\Pi}(x)\gets\sum_{C\in\Pi_P(x)}\mathbf b_C$
\Statex \textbf{Select and admit}
\State Compute $(\boldsymbol\pi,s,\widehat k)$ using \Cref{eq:pact_cumulative_posterior,eq:pact_prediction_vacuity}
\State $y^{(0)}\gets\mathcal A_{\tau}(\widehat k,s,\zeta_{\mathrm{base}}(x))$
\State \Return $Y(x)\gets\mathcal V(y^{(0)},x)$
\end{algorithmic}
\end{algorithm}

% ===================================================================
% SECTION 3.2: Evidence, provenance, and the identifiability requirement
% ===================================================================
\subsection{Evidence, provenance, and the identifiability requirement}
\label{subsec:pre_execution_interface}

Let $\mathcal K=\{k_1,\ldots,k_K\}$, with $K\geq2$, be a finite action-contract set equipped with a fixed tie-breaking order. Each contract pairs a candidate action with the constraint that governs its release, and admission determines whether the selected contract is eligible for execution. Compound labels such as \emph{hold--confirm} name action contracts and are distinct from the \textsc{hold}, \textsc{confirm}, and \textsc{fallback} responses of typed admission.

For a decision $x$, let $\mathcal S(x)=(\xi_1(x),\ldots,\xi_J(x))$ denote the $J\geq1$ expected sources. Each expected source $i$ is associated with a categorical opinion $\mathbf p_i\in\Delta^{K-1}\cup\{\mathbf0\}$, quality and conflict scores $q_i,d_i\in[0,1]$, unavailability and numerical-validity indicators $z_i,v_i\in\{0,1\}$, and a supplied finite nonempty parent set $P_i(x)$. The indicator $z_i=1$ denotes a source that is unavailable, whereas $v_i=1$ indicates that all required numerical attributes are present and satisfy their validity constraints. The evidence vector is determined by the source-local attributes, while the supplied counting structure is determined by the relations among parent sets. The adapter $\Phi_i$ may depend on the role of a source, and timing together with other action-context attributes enters admission separately.

A source-specific adapter $\Phi_i$ maps the local source representation to a nonnegative evidence vector,
\[
\mathbf e_i(x)=\Phi_i\!\left(\xi_i(x)\right)\in\mathbb R_+^K.
\]
Reliability-scaled categorical evidence is used in the simulation benchmark, whereas the external studies employ the evidence mappings specified for their respective evaluations.

Let $\mathcal I=\{1,\ldots,J\}$ denote the expected source indices. The observed, numerically valid subset is
\begin{equation}
\mathcal I_{\mathrm{obs}}(x)
:=
\left\{
i:\ z_i=0,\ v_i=1,\ \mathbf p_i\in\Delta^{K-1}
\right\}.
\label{eq:observed_source_set}
\end{equation}

For $i\in\mathcal I_{\mathrm{obs}}(x)$, let $h_i(x)=\arg\max_{k\in\mathcal K}p_{ik}$, with ties resolved by the order of $\mathcal K$. For a candidate $k$, define the supporting-source set
\begin{equation}
\mathcal I_k(x)
:=
\left\{
i\in\mathcal I_{\mathrm{obs}}(x):h_i(x)=k
\right\}.
\label{eq:supporting_source_set}
\end{equation}

The admission policies in \Cref{subsec:admission_evaluation} use $\mathcal I_k(x)$ to identify the observed sources whose highest-probability contract is $k$. Membership indicates categorical support, while the magnitude of evidence is determined separately by the evidence mapping and by component aggregation.

The counting relation is supplied by acquisition or derivation records. Sources whose parent sets overlap are linked, and the connected components of this relation form the provenance components used in PACT. Numerical predictions are not used in order to infer the relation. Where direct acquisition or derivation records are unavailable, checkpoint or model-family labels are evaluated as surrogate grouping hypotheses in the external studies.

Formally, parent-set overlaps induce the provenance partition
\begin{equation}
G_{\mathrm{prov}}(x)
:=
\bigl(\mathcal I,\mathcal E_P(x)\bigr),
\qquad
(i,j)\in\mathcal E_P(x),\ i\neq j
\Longleftrightarrow
P_i(x)\cap P_j(x)\neq\varnothing,
\qquad
\Pi_P(x)=\operatorname{CC}\!\left(G_{\mathrm{prov}}(x)\right).
\label{eq:provenance_partition}
\end{equation}

The operator $\operatorname{CC}$ returns the connected components. A provenance component is defined by connectivity rather than by direct parent overlap alone, so an indirect bridge may place sources without any direct parent overlap in the same component. Every expected source remains a vertex, including sources that are unavailable or numerically invalid, and the parent links supplied for it are retained. Distinct components are the units that are permitted to accumulate separately under the supplied relation, and they are not certificates of statistical or sensor independence.

The resolution of the supplied relation is determined by parent granularity. When a single acquisition defines a shared origin, the outputs derived from that observation inherit a common parent. A later acquisition receives a distinct parent when it is intended to remain separately countable, unless some other recorded relation connects the two.

The induced partition fixes the counting units. Two further assumptions govern how the evidence budgets of these units are compared and combined.

\begin{assumption}[Commensurate evidence units]
\label{assump:commensurate_evidence}
All evidence vectors combined within a decision instance are expressed in commensurate additive units.
\end{assumption}

\begin{assumption}[Separate-component additivity]
\label{assump:separate_component_additivity}
For distinct provenance components represented in commensurate evidence units, their retained budgets combine additively.
\end{assumption}

Commensurability concerns the scale of the evidence, whereas separate-component additivity governs accumulation across distinct provenance components. Neither assumption asserts statistical independence. Commensurability is enforced by construction in the benchmark. The external studies either analyze each dataset--model pair under a fixed evidence mapping or adopt a common explicit mapping whenever multiple checkpoints are combined.

For the benchmark, the source reliability factor $\rho_i$ discounts the categorical opinion and $\kappa_i>0$ sets the scale of its evidence. The resulting evidence mapping is
\begin{equation}
\mathbf e_i
:=
\begin{cases}
\rho_i\kappa_i\mathbf p_i,
&
z_i=0,\ v_i=1,\ \mathbf p_i\in\Delta^{K-1},
\\[2pt]
\mathbf 0,
&
\text{otherwise}.
\end{cases}
\label{eq:pact_evidence_adapter}
\end{equation}

Appendix~\ref{sec:proofs} gives the benchmark construction of $\rho_i$. A component $C$ is complete for an instance $x$ when $C\subseteq\mathcal I_{\mathrm{obs}}(x)$. Under the primary expected-source convention, expected sources that are unavailable or invalid remain in their components as zero vectors, so that the coordinatewise meet of any incomplete component is zero (\Cref{eq:pact_component_budget}). Separate analyses compare this convention with alternative treatments of missing sources.

\begin{proposition}[Value-only indistinguishability]
\label{prop:opinion_only_indistinguishability}

Let $\mathbb A$ be a space of source-local numerical attributes that excludes provenance parent sets. Its evidence coordinate lies in $\mathbb R_+^K$ and may be accompanied by reliability, conflict, availability, and validity. Let $\mathcal M(\mathbb A)$ denote the finite multisets over $\mathbb A$, and let $H:\mathcal M(\mathbb A)\to\mathcal Y$ be any value-only fusion map, where $\mathcal Y$ is an arbitrary output space. Consider two inputs that present the same numerical attribute multiset $\{\boldsymbol\alpha,\boldsymbol\alpha\}$, where the evidence coordinate of $\boldsymbol\alpha$ is a nonzero vector $\mathbf e$. In one input, the second source is a duplicate that carries the parent set of the first. In the other, the two identically valued sources carry disjoint parent sets and therefore lie in separate components. No value-only map can distinguish these inputs, since both present $\{\boldsymbol\alpha,\boldsymbol\alpha\}$ to $H$ and therefore receive the same output.

\end{proposition}

Reliability, conflict, availability, and validity can regulate the strength of a contribution, yet identical source-local values cannot distinguish the two counting structures in \Cref{prop:opinion_only_indistinguishability}. Provenance supplies relational information concerning how inputs participate in acquisitions or derivations~\citep{Green2007ProvenanceSemirings}, and that supplied relation is what defines the counting partition in PACT. Multiple supporting outputs may therefore correspond to fewer separately countable units even where local reliability is high and conflict is low.

% =========================================================
% SECTION 3.3: Provenance-conserving fusion and guarantees
% =========================================================
\subsection{Provenance-conserving fusion and guarantees}
\label{subsec:pact_fusion}

Given the supplied partition, singleton fidelity preserves the evidence of a singleton, and insertion non-amplification prevents an added member from increasing the budget of a component. These requirements bound the evidence that each component may retain.

Equip $\mathbb R_+^K$ with the coordinatewise partial order $\mathbf r\preceq\mathbf s$ if and only if $r_k\leq s_k$ for every $k\in\mathcal K$. Starting from any member $\mathbf e_i$ as a singleton and inserting the remaining members shows that the output $\mathbf g_C$ of any admissible rule must satisfy $\mathbf g_C\preceq\mathbf e_i$ for every $i\in C$. The coordinatewise meet is the greatest vector that satisfies all of these inequalities. A bound on the total evidence mass alone would permit support to shift across contracts and would not enforce the same coordinatewise constraints.

For each provenance component $C\in\Pi_P(x)$, define
\begin{equation}
\mathbf b_C
:=
\bigwedge_{i\in C}\mathbf e_i,
\qquad
(\mathbf b_C)_k
:=
\min_{i\in C}e_{ik}.
\label{eq:pact_component_budget}
\end{equation}
The meet itself satisfies singleton fidelity and insertion non-amplification, so this pointwise bound is attained within the admissible class of rules. The resulting budget need not reproduce any individual opinion.

\begin{proposition}[Within-component characterization]
\label{prop:pact_within_component_characterization}
Let $g$ map every finite nonempty multiset $\mathcal E$ of vectors in $\mathbb R_+^K$ to $\mathbb R_+^K$, and let $\uplus$ denote multiset insertion. Among all rules satisfying singleton fidelity, $g(\{\mathbf e\})=\mathbf e$, and insertion non-amplification, $g(\mathcal E\uplus\{\mathbf e\})\preceq g(\mathcal E)$, the coordinatewise meet is the unique pointwise greatest rule.
\end{proposition}

Exact-copy invariance alone does not identify the meet, since the coordinatewise maximum is likewise invariant under duplicate insertion. Insertion non-amplification prevents even a nonidentical member with predictive value from increasing the retained evidence of an existing component. This requirement is adopted in PACT as an accounting requirement within a component, not as a criterion of predictive optimality.

Across components, the retained budgets are added under \Cref{assump:commensurate_evidence,assump:separate_component_additivity}. For the extensions of the source set considered below, the existing evidence vectors and parent assignments are held fixed.

\begin{proposition}[Greatest admissible provenance-aware evidence budget]
\label{prop:pact_common_budget}
Under \Cref{assump:commensurate_evidence,assump:separate_component_additivity}, consider cumulative budgets of the form $\sum_{C\in\Pi_P(x)}\mathbf g_C$, where each $\mathbf g_C\in\mathbb R_+^K$ satisfies $\mathbf g_C\preceq\mathbf e_i$ for every $i\in C$. The unique greatest budget in this admissible class is
\begin{equation}
\sum_{C\in\Pi_P(x)}\mathbf b_C
=:
\mathbf E_{\Pi}(x).
\label{eq:pact_cumulative_evidence}
\end{equation}
Adding a source that joins one or more existing components cannot increase $\mathbf E_{\Pi}$ coordinatewise. An exact duplicate carrying the original provenance parent set leaves $\mathbf E_{\Pi}$ invariant. By contrast, a source forming a disconnected component $D$ gives $\mathbf E_{\Pi}^{+}=\mathbf E_{\Pi}+\mathbf b_D\succeq\mathbf E_{\Pi}$.
\end{proposition}

Provenance-aware evidence conservation is defined in \Cref{def:provenance_evidence_conservation}. Which outputs may be counted separately is specified by the supplied partition, while the meet constrains the budget they retain without estimating latent shared information and without separating shared from source-specific contributions.

\begin{corollary}[Partition monotonicity]
\label{cor:pact_partition_monotonicity}
For a given source index set and evidence collection, let $\Pi_2$ be a coarsening of $\Pi_1$. Then
\begin{equation}
\mathbf E_{\Pi_2}\preceq\mathbf E_{\Pi_1}.
\label{eq:pact_partition_monotonicity}
\end{equation}
Coarsening the provenance partition cannot increase cumulative evidence. Refinement may increase it.
\end{corollary}

Partition monotonicity holds the source index set and numerical evidence fixed. Perturbation analysis instead fixes the source set, partition, prior strength, and base rate. Bounded perturbations of the source evidence yield bounded changes in the cumulative budget, posterior, and selection score. With threshold and base eligibility fixed, sufficient margins preserve the predicted contract and score-selection decision. Appendix~\ref{sec:proofs} gives the bounds and margin conditions.

Near-copy insertion enlarges the source set without merging distinct components. Appendix~\ref{sec:proofs} bounds the change in the receiving component's budget, which cannot increase. Stable predictions and selection scores alone do not ensure unchanged typed responses, which also depend on the other admission inputs.

% ==============================================
% SECTION 3.4: Selection scores and typed admission
% ==============================================
\subsection{Selection scores and typed admission}
\label{subsec:selective_admission_method}

Fusion quantifies support, whereas admission determines whether a candidate may proceed to execution. In the controlled benchmark, an action contract is selected by the fused posterior, and the retained evidence mass ranks candidates for selective retention. Retained candidates then undergo the ordered checks of typed admission.

Let $B_{\Pi}(x):=\lVert\mathbf E_{\Pi}(x)\rVert_1$, let $\mathbf a\in\Delta^{K-1}$ be a base-rate vector, and let $W>0$ be the total prior strength. All experiments use the uniform base rate $a_k=1/K$ and one unit of prior strength per class, so that $W=K$. The cumulative evidence is projected into the categorical distribution
\begin{equation}
\boldsymbol\pi(x)
:=
\frac{W\mathbf a+\mathbf E_{\Pi}(x)}
{W+B_{\Pi}(x)}.
\label{eq:pact_cumulative_posterior}
\end{equation}

The predicted contract, the vacuity, and the selection score of PACT are
\begin{equation}
\widehat k(x)
:=
\arg\max_{k\in\mathcal K}\pi_k(x),
\qquad
u(x)
:=
\frac{W}
{W+B_{\Pi}(x)},
\qquad
s(x)=1-u(x),
\label{eq:pact_prediction_vacuity}
\end{equation}
with ties in $\widehat k$ resolved by the order of $\mathcal K$. Taken together, \Cref{eq:pact_component_budget,eq:pact_cumulative_evidence,eq:pact_cumulative_posterior,eq:pact_prediction_vacuity} specify the fusion map $F_{\Pi}$ and the predicted contract. An exact copy carrying the original evidence vector and parent set leaves these outputs unchanged. For fixed $W$, the score $s=1-u$ is monotone in $B_{\Pi}$ and therefore ranks instances by retained evidence mass. This is a score of evidence mass rather than a calibrated probability of correctness, and its ability to rank candidates by correctness is evaluated empirically.

Two limiting partitions clarify the role of provenance. If every source forms its own component, then $\mathbf E_{\Pi}=\sum_i\mathbf e_i$. If all sources share one component, then $\mathbf E_{\Pi}=\bigwedge_i\mathbf e_i$. With $W=K$, a uniform base rate, and an all-singleton partition, $\boldsymbol\pi$ and $u$ reduce to the summed-evidence Dirichlet projection and vacuity~\citep{Sensoy2018EDL}. The aggregation structure between these two limiting cases is determined by the supplied partition.

In the controlled benchmark, score-based selection applies a threshold $\tau\in[0,1]\cup\{+\infty\}$ to the method-specific score together with a common eligibility condition $\zeta_{\mathrm{base}}(x)\in\{0,1\}$. The candidate $\widehat k$ is selected when $s\geq\tau$ and $\zeta_{\mathrm{base}}(x)=1$, and is withheld otherwise. The compared fusion rules share the same base-eligibility requirement. \Cref{subsec:admission_evaluation} specifies $\zeta_{\mathrm{base}}$, and Appendix~\ref{sec:proofs} gives the exact selection map $\mathcal A_{\tau}$.

Typed admission evaluates the release conditions that are not represented by the selection score. In the benchmark path, a candidate withheld by score-based selection remains withheld. A retained candidate is tested for command consistency, source validity, risk support, and component corroboration, in that order, with the corresponding failure responses \textsc{hold}, \textsc{fallback}, \textsc{hold}, and \textsc{confirm}. The response is determined by the first failed check, while a candidate that passes all four checks is admitted. Appendix~\ref{sec:proofs} defines the corresponding failure indicators and the case map $\mathcal V$.

Writing $y^{(0)}(x)=\mathcal A_{\tau}\!\left(\widehat k(x),s(x),\zeta_{\mathrm{base}}(x)\right)$ for the preliminary selection decision, the complete benchmark path in Algorithm~\ref{alg:pact_decision_path} is the composition
\begin{equation}
\begin{aligned}
(\boldsymbol\pi(x),s(x))
&=
F_{\Pi_P(x)}\!\left((\mathbf e_i(x))_{i=1}^{J}\right),
\qquad
\widehat k(x)=\arg\max_{k\in\mathcal K}\pi_k(x),
\\
Y(x)
&=
\mathcal V\!\left(
\mathcal A_{\tau}\!\left(\widehat k(x),s(x),\zeta_{\mathrm{base}}(x)\right),
x
\right).
\end{aligned}
\label{eq:complete_decision_map}
\end{equation}

An exact copy inserted within a component leaves the complete decision in \Cref{eq:complete_decision_map} unchanged, provided that it inherits the original evidence vector and parent set and alters neither base eligibility nor any other admission input.

The offline HRC studies in Section~\ref{subsec:habit_target_temporal_transfer} use the same fusion operator together with application-specific candidate, target-identity, and temporal requirements. No evidence-mass cutoff is applied in these HRC studies. The score-thresholded procedure described above applies to the controlled benchmark.

% ==================================================================
% SECTION 3.5: Computational complexity and provenance misspecification
% ==================================================================
\subsection{Complexity and provenance misspecification}
\label{subsec:selection_properties}

With bounded parent-set cardinality, naive pairwise discovery of parent overlaps followed by evidence aggregation has an expected $O(J^2+JK)$ upper bound on time for the stated hash-based implementation. Once the provenance graph is available in sparse form, connected-component traversal and aggregation require $O(J+\lvert\mathcal E_P(x)\rvert+JK)$ time and space. Appendix~\ref{sec:proofs} gives the full bounds and the corresponding assumptions on representation.

Relative to a reference grouping, false refinement separates members of one component, while over-coarsening merges separately countable units. An unrecorded shared parent can cause false refinement. An erroneous overlap or transitive bridge can cause over-coarsening. With numerical evidence fixed, refinement cannot reduce cumulative evidence and coarsening cannot increase it (\Cref{cor:pact_partition_monotonicity}). \Cref{subsec:results_provenance_analysis} tests their predictive and admission consequences, which these budget directions alone do not determine.

% ============================================================
% SECTION 4: Experimental Design
% ============================================================
\section{Experimental Design}
\label{sec:bench_construction}

The controlled benchmark examines whether evidence is preserved under exact-copy insertion in PACT and whether the response to fixed-output provenance reassignments follows the predicted structural direction. It further measures how such changes affect prediction, selective retention, calibration, and typed admission. Tests on learned multi-view and multi-camera outputs examine whether the predicted response of the budget persists, and the offline HRC study then evaluates the consequences for admission.

% ============================================================
% SECTION 4.1: Evaluation logic and benchmark design
% ============================================================
\subsection{Matched-reassignment strategy and benchmark design}
\label{subsec:experimental_rqs}

Source values and provenance assignments are allowed to vary separately in the controlled benchmark, which comprises 48 NVIDIA Isaac Sim scenes~\citep{NVIDIAIsaacSim} balanced across three handover geometries. Each decision combines language ($L$), geometry ($G$), and path-risk ($R$) opinions over five action contracts $\mathcal K=\{N,S,H,T,U\}$. The reference provenance partition $\{\{L\},\{G,R\}\}$ assigns language to one component and groups geometry with risk through a shared scene-context parent.

The same pre-perturbation label serves as the evaluation reference and also initializes the language opinion. Language-only controls, evaluated with and without base eligibility, measure the effect of using this label for both purposes. Evidence accounting, selective retention, and admission are tested in the benchmark using constructed opinions. \Cref{fig:curated_benchmark_examples} shows how a candidate is evaluated, and \Cref{tab:benchmark_definition} defines the action contracts.

\begin{figure}[pos=t]
\includegraphics[width=\textwidth]{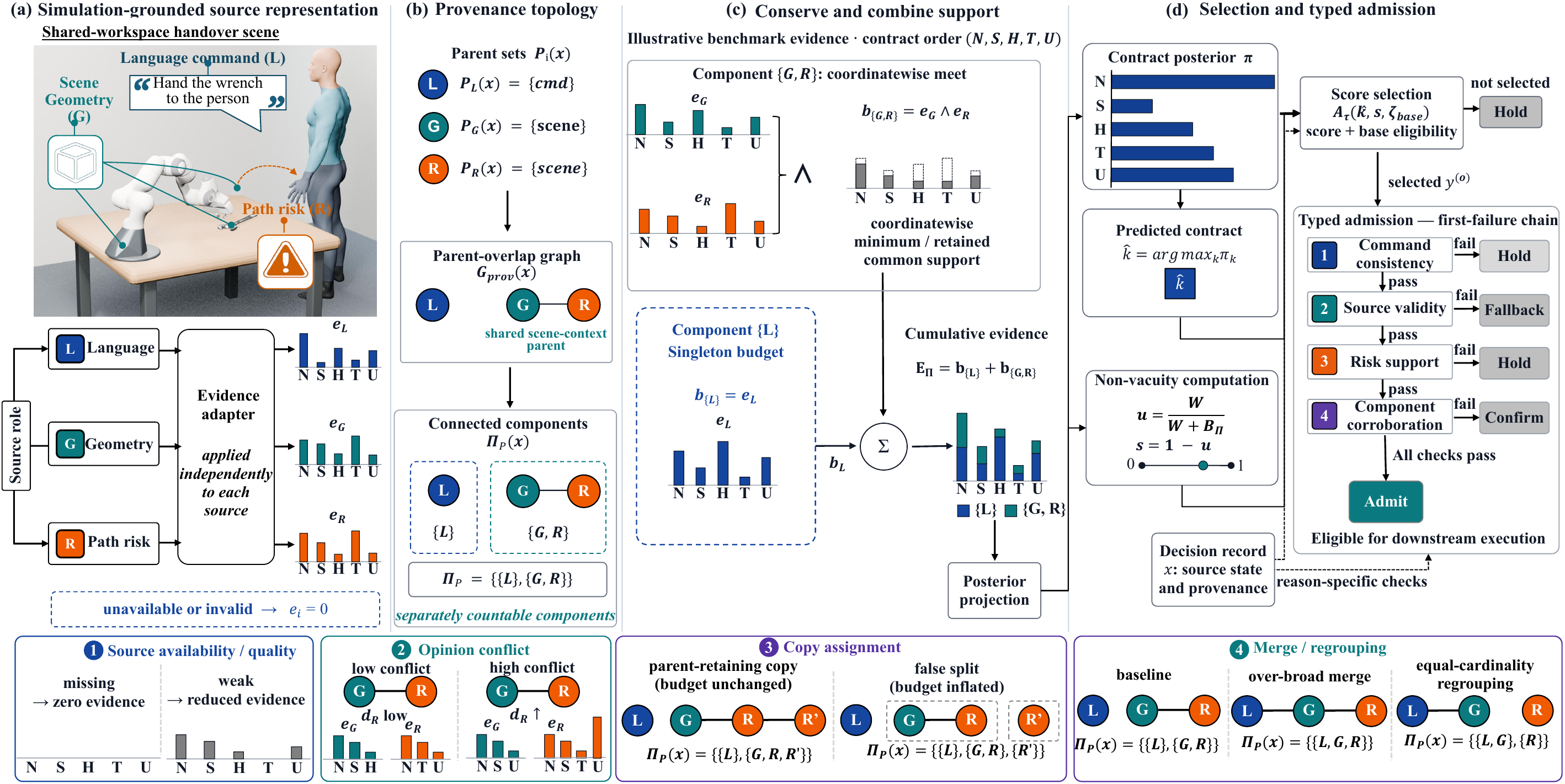}
\caption{PACT architecture and paired interventions in the simulation benchmark. (a) Language, geometry, and path-risk outputs are mapped to classwise evidence. Unavailable or invalid sources contribute zero evidence. (b) Parent-set overlap induces the provenance graph and the connected components $\{\{L\},\{G,R\}\}$. (c) The coordinatewise meet within $\{G,R\}$ is retained, the singleton budget $\{L\}$ is preserved, and the result is summed across separately countable components. (d) The posterior determines the action contract, the score $1-u$ orders instances for retention, and the ordered admission checks stop at the first failed requirement. The bottom row shows the paired intervention families. Admission denotes eligibility for execution.}
\label{fig:curated_benchmark_examples}
\end{figure}

\begin{table}[pos=t]
\centering
\caption{Action contracts used in the simulation benchmark. Contract names denote candidate action constraints and are distinct from the typed admission responses \textsc{hold}, \textsc{confirm}, and \textsc{fallback}.}
\label{tab:benchmark_definition}
\IFTableSetup
\setlength{\tabcolsep}{3.8pt}
\renewcommand{\arraystretch}{1.08}
\begin{tabularx}{\linewidth}{@{}M{1.05cm} Q{2.15cm} Y Q{1.85cm}@{}}
\toprule
\textbf{Symbol} & \textbf{Action contract} & \textbf{Execution constraint} & \textbf{Benchmark label} \\
\midrule
$N$ & Normal & Execute the nominal transfer without an added constraint. & Neutral \\
$S$ & Slow-clearance & Proceed at reduced speed with additional clearance. & Cautious \\
$H$ & Hold--confirm & Remain stationary while requesting or awaiting confirmation. & Uncertain \\
$T$ & Retreat--fallback & Withdraw from the transfer or invoke a fallback before reattempting. & Corrective \\
$U$ & Bounded--urgent & Act within defined time, resource, and state bounds. & Time-critical \\
\bottomrule
\end{tabularx}
\end{table}

The benchmark contains 31,200 evaluations nested within 48 scene clusters. Thirteen source conditions vary availability, validity, quality, conflict, timing, and urgency while the reference task structure is preserved. Of these evaluations, 19,200 involve three valid sources, 7,200 involve exactly one unavailable source, 2,400 involve three unavailable sources, and 2,400 involve one invalid source. Resampling of the 31,200 evaluations is performed at the scene level.

One condition constitutes a designed adversarial-consensus stress test. Its 2,400 cases present high-confidence agreement on an incorrect contract, while the specified corroboration rule recognizes fewer separately countable supporting components than admission requires. The remaining conditions test evidence degradation and disagreement without recourse to this constructed failure mode. Appendix~\ref{sec:additional_experimental_details} gives the complete perturbation and comparator specifications.

Each intervention targets a distinct part of the accounting mechanism. Exact and near copies test conservation within a provenance component. False refinement, over-coarsening, bridging, and regrouping alter the counting relation while the numerical evidence is kept fixed. Missing-source variants change the manner in which declared but unavailable evidence enters the component rule. For multiplicity transfer, an ordering score fitted at the baseline multiplicity is applied without refitting once additional copies have been introduced.

In the main grouped-scene study, nested-Dirichlet concentration and score cutoffs are selected only within each outer training fold. At the target coverages $0.10$, $0.13$, and $0.15$, each concentration is evaluated using its corresponding nested-Dirichlet candidates and score cutoff under the shared admission rule. Both cutoff fitting and cost evaluation rely on outer-training data. Concentrations are ranked according to the 10:1 wrong-admission-to-nonrelease cost in \Cref{eq:decision_cost}. This ratio is illustrative rather than estimated from an application-specific utility, and it is also used in the policy-cost comparison. The selected concentration is used unchanged in PACT, so that the corresponding evidence scale is preserved. Appendix~\ref{sec:additional_experimental_details} provides the cost definition, the selection procedure, and the tie-breaking rules.

A separate 256-point scrambled Sobol design~\citep{Owen1995RandomizedNets} evaluates sensitivity to opinion shape, generation concentration, quality, conflict, and occlusion strength.

% ============================================================
% SECTION 4.2: Admission, comparators, and evaluation
% ============================================================
\subsection{Admission, comparators, and evaluation}
\label{subsec:admission_evaluation}

In the controlled benchmark, the complete fusion methods operate on the same source instances under a common typed admission policy. They may differ in candidate formation, evidence representation, missing-source treatment, access to provenance at the fusion stage, native selection score, and fitted cutoff. These comparisons assess complete decision rules rather than isolated causal effects of the fusion operator. All methods use the same supplied provenance partition for component corroboration at the admission stage.

The base-eligibility indicator is $\zeta_{\mathrm{base}}(x)=\mathbb 1\{|\mathcal I_{\mathrm{obs}}(x)|\geq2\}$. At admission, $k^*=\widehat k(x)$ denotes the selected candidate, and $\mathcal U_3(k^*;x)$ indicates that all three expected sources are observed and support $k^*$. A component is complete when all of its expected members are observed and well formed, and each complete component counts once whenever at least one member supports the selected candidate. The number of complete supporting components is
\[
n_{\mathrm{comp}}^{\mathrm{complete}}(k^*;x)=
\left|
\left\{
C\in\Pi_P(x):
C\cap\mathcal I_{k^*}(x)\neq\varnothing,\ 
C\subseteq\mathcal I_{\mathrm{obs}}(x)
\right\}
\right|.
\]

This count of components is distinct from the magnitude of evidence derived from the meet. Appendix~\ref{sec:additional_experimental_details} defines the source-level probability, quality, and conflict summaries used in \Cref{tab:typed_admission_policy}.

The checks in \Cref{tab:typed_admission_policy} begin after base eligibility and score selection and stop at the first failed condition, with component corroboration as the final check. Under the primary rule, the requirement on component count is invoked only when $\mathcal U_3(k^*;x)$ and the listed probability, quality, and conflict conditions are satisfied. In that case, fewer than $\nu=3$ complete supporting components trigger \textsc{confirm}. The reference partition $\{\{L\},\{G,R\}\}$ contains only two components, so high-confidence unanimous support that meets these conditions is sent for confirmation whenever this final check is reached. Alternative policies vary the required number or combination of complete supporting components. The effects of those policies are evaluated in \Cref{subsec:results_verifier_source_support}, and \Cref{tab:component_corroboration_rules} reports the corresponding outcomes.

\begin{table}[pos=t]
\centering
\caption{Typed action admission and ordered failure responses.}
\label{tab:typed_admission_policy}
\begin{threeparttable}
\IFTableSetup
\begingroup
\footnotesize
\renewcommand{\arraystretch}{1.02}
\setlength{\tabcolsep}{2.6pt}
\begin{tabularx}{\linewidth}{@{}
>{\raggedright\arraybackslash}p{2.35cm}
>{\raggedright\arraybackslash}X
>{\centering\arraybackslash}p{1.25cm}@{}}
\toprule
Admission check & Observable condition & Response \\
\midrule
Command consistency
& Timing information is missing for a present language source, the current command differs from that underlying the evidence, or a stale-language state carries over the previous command
& \textsc{hold} \\
Source validity
& A present source has an invalid opinion vector, an unrecognized contract label, or a missing required opinion, quality, or conflict attribute
& \textsc{fallback} \\
Risk support
& $z_R=1$, $z_L=0$, $p_L^{\mathrm{peak}}\geq0.90$, $q_L\geq0.90$, and $\Delta_L\geq0.70$
& \textsc{hold} \\
Component corroboration
& $\mathcal U_3(k^{*};x)$, $p_{\min}^{\mathrm{peak,obs}}\geq0.40$, $q_{\min}^{\mathrm{obs}}\geq0.30$, $d_{\max}^{\mathrm{obs}}\leq0.15$, $n_{\mathrm{comp}}^{\mathrm{complete}}(k^{*};x)<\nu$, with $\nu=3$
& \textsc{confirm} \\
\bottomrule
\end{tabularx}
\endgroup
\begin{tablenotes}[flushleft]\footnotesize
\item[] \textit{Note.} Checks are evaluated after base eligibility and score selection, in the displayed order. The first failed condition determines the typed response.
\end{tablenotes}
\end{threeparttable}
\end{table}

The comparator set covers quality weighting, product fusion, evidential composition, cautious combination, provenance-conditioned alternatives, and a learned Set Transformer reference~\citep{Han2021TMC,Denoeux2008Cautious,SmetsKennes1994TBM,Lee2019SetTransformer}. \Cref{tab:comparator_specification} summarizes the analytic fusion rules. Two targeted ablations address properties of the component rule used in PACT. The first examines whether exact-copy invariance alone suffices for evidence conservation. The second retains the first source in each component under the common source order and examines whether a single representative recovers the coordinatewise meet.

\begin{table}[pos=t]
\centering
\caption{Compared fusion rules by evidence form, provenance use, missing-source treatment, and selection score.}
\label{tab:comparator_specification}
\begin{threeparttable}
\IFTableSetup
\begingroup
\footnotesize
\renewcommand{\arraystretch}{1.00}
\setlength{\tabcolsep}{2.4pt}
\begin{tabularx}{\linewidth}{@{}
>{\raggedright\arraybackslash}p{2.65cm}
>{\raggedright\arraybackslash}p{1.75cm}
>{\raggedright\arraybackslash}X
>{\centering\arraybackslash}p{0.70cm}
>{\centering\arraybackslash}p{1.45cm}
>{\centering\arraybackslash}p{1.05cm}@{}}
\toprule
Method & Evidence form & Combination rule & Prov. &
\makecell[c]{Missing\\source} & \makecell[c]{Selection\\score} \\
\midrule
Quality-weighted & Probability & Quality-weighted mean & \xmark & Omitted & Peak \\
Product & Probability & Reliability-discounted product & \xmark & Omitted & Peak \\
Nested Dirichlet & Dirichlet evidence & Sequential composition & \xmark & Omitted & $1-u$ \\
Global cautious & Belief weights & Global cautious minimum & \xmark & Omitted & Peak \\
Hierarchy-matched cautious & Belief weights & Component minimum, then conjunctive combination & \cmark & Vacuous mass & $1-m(\Theta)$ \\
Pooling without provenance & Log probability & Weighted log pool & \xmark & Penalized & Peak \\
Provenance-discounted pooling & Log probability & Edge-discounted log pool & \cmark & Penalized & Peak \\
\textbf{PACT fusion} & Classwise evidence & Meet, then sum & \cmark & Zero evidence & $1-u$ \\
\bottomrule
\end{tabularx}
\endgroup
\begin{tablenotes}[flushleft]\footnotesize
\item[] \textit{Note.} \cmark\ indicates that the supplied provenance partition enters the fusion rule. Prov. denotes use of provenance in fusion. Component corroboration uses the same supplied partition across all complete decision rules.
\end{tablenotes}
\end{threeparttable}
\end{table}

Native selection scores are evaluated alongside posterior-derived scores and a pooled correctness scorer. Within each outer fold, a single pooled correctness model is fitted only on outer-training outputs from product fusion, nested Dirichlet, global cautious fusion, and PACT, without method identity. The fitted model and the standardization of its training fold are then applied without refitting to provenance-discounted pooling and to hierarchy-matched cautious fusion. The shared scorer may produce different numerical scores and different rankings for each method. Cutoffs are fitted separately on the outer-training instances of each method, so that the retained sets may differ as well. Appendix~\ref{sec:additional_experimental_details} gives the comparator and scorer specifications.

Pre- and post-admission score cutoffs are fitted separately within each method, using outer-training instances that are eligible at the corresponding stage, and are transferred unchanged to the held-out scenes. The resulting risk--coverage curves therefore describe separately thresholded decision rules rather than the same retained set before and after admission.

Coverage is measured relative to all evaluation instances, and the reported operating point targets a coverage of $0.13$. Among $N$ instances, let $n_{\mathrm{ret}}$, $n_{\mathrm{err}}$, and $n_{\mathrm{cor}}$ count the retained decisions, the retained errors, and the correct retained decisions. Then $\gamma=n_{\mathrm{ret}}/N$, $R_{\mathrm{all}}=n_{\mathrm{err}}/N$, $R_{\mathrm{cond}}=n_{\mathrm{err}}/n_{\mathrm{ret}}$ when $n_{\mathrm{ret}}>0$, and $C_{\mathrm{all}}=n_{\mathrm{cor}}/N=\gamma-R_{\mathrm{all}}$. The principal metric, ncsAURC, averages the conditional risk over the shared coverage~\citep{ElYaniv2010SelectiveClassification}.
\begin{equation}
\operatorname{ncsAURC}_{[\gamma_{\min},\gamma_{\max}]}
=
\frac{1}{\gamma_{\max}-\gamma_{\min}}
\int_{\gamma_{\min}}^{\gamma_{\max}}R_{\mathrm{cond}}(\gamma)\,\mathrm d\gamma.
\label{eq:common_support_ncsaurc}
\end{equation}

A lower value of ncsAURC indicates lower conditional risk averaged over the shared coverage interval. Cross-method areas are integrated only over the coverage attained by every compared method, which avoids extrapolation beyond the observed support. The primary comparison uses $[0.10,0.39]$, whereas repeated fold assignments, the joint holdout, and source-evidence perturbations use $[0.10,0.35]$.

A score-independent random ordering provides the reference used in \Cref{subsec:results_selective_fusion}. Its expected conditional risk equals the empirical error rate among base-eligible instances. Appendix~\ref{sec:additional_experimental_details} gives the empirical denominator and the numerical integration procedure.

NLL and the multiclass Brier score assess probabilistic prediction through strictly proper scoring rules~\citep{Gneiting2007ProperScores}. ECE$_{10}$ summarizes top-confidence calibration~\citep{Guo2017Calibration}, using posterior peak and ten equal-width bins weighted by empirical frequency. The unnormalized Brier score lies in $[0,2]$. Appendix~\ref{sec:additional_experimental_details} gives the metric definitions.

An exploratory matched $2\times2$ comparison crosses singleton against provenance-partition aggregation, with component corroboration either enabled or disabled. The evidence mapping, the score functional, and the remaining admission conditions are held fixed across the four cells, although numerical scores and rankings may still differ. The resulting cells provide matched contrasts between complete decision rules and do not separate the causal contributions of partitioning and of corroboration. Since the adversarial-consensus condition is constructed so as to activate corroboration, a sensitivity analysis removes its 2,400 cases and evaluates the remaining 28,800.

Model selection, score fitting, threshold estimation, and uncertainty analyses all preserve scene grouping. Admission rules are fixed before held-out evaluation, and grouping assignments are defined independently of the held-out outcomes. The main grouped-scene analyses separate training from test data by scene fold. In the joint scene-and-condition holdout, score-cutoff fitting excludes both the outer test-scene fold and the held condition, whereas the fusion-concentration schedule is transferred from the main study without reselection for each held condition.

Sensitivity to scene allocation is assessed over 50 grouped fold assignments, with scenes remaining the sampling units. Paired bootstrap intervals resample the 48 held-out scene contributions. Appendix~\ref{sec:additional_experimental_details} gives the remaining numerical and uncertainty-analysis details.

% ============================================================
% SECTION 4.3: Evaluation on learned predictions
% ============================================================
\subsection{Matched reassignment on learned predictions}
\label{subsec:external_transfer_design}

The reassignment analyses vary the supplied counting structure while the learned predictions are held fixed. Changes in the budget and in the posterior therefore reflect evidence accounting rather than retraining or altered source outputs. Trusted Multi-View Classification (TMC)~\citep{Han2021TMC} and Reliable Conflictive Multi-View Learning (RCML)~\citep{Xu2024RCML} provide the per-view predictions for these tests, and the offline HRC analyses add camera-acquisition grouping together with application-specific admission evidence.

The multi-view analysis takes the identity of the original feature view as the reference counting assignment and compares the same predictions under false refinement and under all-view merging. The six-view HandWritten dataset, also denoted Mfeat, permits exhaustive analysis of all 203 view partitions and 856 single-merge relations for both TMC and RCML. Each single-merge relation joins two components of a partition while the remaining components are left unchanged. Predictions of TMC on Scene15 provide a second evaluation of accuracy and selective ordering, and predictions of RCML on the Pose, Illumination, and Expression (PIE) dataset extend the analysis of the evidence budget.

The Human-Aware Behavior and Interaction Training (HABIT) dataset provides the robot-manipulation observations. The analysis comprises 1,128 paired events drawn from 696 episodes across six tasks. Observations come from five synchronized cameras, and each image is evaluated with four prompts~\citep{Song2026HABIT}. Prompt outputs derived from the same image form one acquisition component, whereas distinct camera images are treated as separate acquisitions at the granularity evaluated here. This grouping does not assert statistical independence. The comparisons include per-output counting, acquisition grouping, identical-vector merging, all-view merging, and one output per camera.

Since readiness annotations are not provided in HABIT, subtask indices and gripper trajectories define the offline early and event-proximal evaluation windows. These windows determine evaluation times rather than deployment-time inputs. A separate analysis evaluates model-family labels as a surrogate grouping hypothesis by comparing them with the observed association of errors across eight vision--language checkpoints~\citep{Qwen3VL2025,InternVL32025,Li2024LLaVAOneVision,Marafioti2025SmolVLM}.

The HRC admission study uses 60 evaluation episodes that are disjoint from the 82 development episodes. Across two temporal windows, one episode target, and five counterfactual targets, the study contains 720 offline target--time evaluations per checkpoint, with the episode as the sampling unit. In the learned event-proximity interface, candidate evidence and the learned event-proximity signal share camera acquisitions and enter admission as complementary requirements.

For the event-proximity model, the primary evaluation tests unseen episodes from the same six tasks that are represented in development. A leave-one-task-out analysis additionally withholds each evaluated task from fitting. Absolute-threshold and within-episode-change analyses examine task-dependent score levels and within-episode temporal discrimination.

The joint fusion-and-admission design includes camera-grouped PACT, per-output counting, all-output merging, and one output per camera, using predictions of Qwen3-VL-8B and Qwen3-VL-32B obtained from the same 60 episode-level events at two temporal windows. Prompt multiplicity and exact copies vary repeated computation under fixed acquisitions. Camera removal changes acquisition availability and thereby changes the observations presented to the decision layer. The primary outcomes are reference-inconsistent admissions across 720 queries per checkpoint and recall among the 60 reference-\textsc{ready} episode targets. Reference consistency denotes agreement with the dataset reference action.

A separate analysis of the temporal mechanism uses the larger paired-event population rather than the 60-episode admission subset and tests early-to-event-proximal changes against block-exchangeable temporal variation. The permutation analysis excludes the sole singleton exchangeability block. Appendix~\ref{sec:additional_experimental_details} specifies the evidence mapping, the model, the thresholds, the resampling units, and the permutation design.

% ============================================================
% SECTION 5: Results: Conservation, Ordering, and Admission
% ============================================================
\section{Conservation, Ordering, and Admission Results}
\label{sec:admission_results}

Copy and provenance interventions examine whether the accumulation of evidence follows the supplied counting relation. The remaining analyses consider how changes in the budget affect predictions, selective ordering, and admission.

% ============================================================
% SECTION 5.1: Conservation under copy and provenance interventions
% ============================================================
\subsection{Evidence conservation under copy and provenance interventions}
\label{subsec:results_provenance_analysis}

\paragraph{Exact-copy invariance depends on the counting relation.}
No posterior drift arises for PACT or for hierarchy-matched cautious fusion when exact copies remain within their original components. Assigning the same copies to distinct components changes the predicted contract in 32.3\% of 93,600 comparisons, with a mean posterior $\ell_1$ drift of 0.396. These comparisons cover all 31,200 instances, each of the three source roles being duplicated in turn at multiplicity eight. Product fusion and nested Dirichlet do not use the provenance partition, so shared-parent and separate-parent assignments coincide numerically. Appendix~\ref{sec:additional_experimental_details} gives the remaining comparator change rates and drift statistics.

Preservation of the selected contract does not imply stability of the posterior, and neither property determines the ordering induced by the selection score.

\paragraph{Near copies perturb the posterior without adding support.}
Under the undegraded source condition, near-copy insertion changes no predicted contract across 960 comparisons that span 48 scenes, five benchmark labels, and four perturbation magnitudes. The maximum posterior $\ell_1$ drift is 0.011. A near copy inserted within a component may lower the budget of that component but cannot add support. The sweep measures the sensitivity of the posterior to perturbations of the opinions. Since the formal bounds concern perturbations in evidence space, the sweep does not test their numerical values directly. Appendix~\ref{sec:additional_experimental_details} reports the perturbation grid and the descriptive slope.

\paragraph{Exact-copy invariance alone does not characterize conservation.}
The coordinatewise maximum is likewise invariant under exact duplication and satisfies partition monotonicity, yet the insertion of a different member can raise the budget of its component. This occurs in 96.4\% of the 24,000 instances in which geometry and risk are observed and valid, spanning ten of the thirteen source conditions. The corresponding rate for the meet is 0\%.

Replacing the meet with the maximum raises the total evidence budget $B_{\Pi}$ by 35.4\% across the full benchmark and changes 16.2\% of the predicted contracts. It also raises ncsAURC from 0.6294 to 0.7918 (difference 0.1623, 95\% confidence interval (CI) $[0.154,0.171]$), whereas ECE$_{10}$ decreases from 0.245 to 0.193. A lower calibration error therefore does not imply compliance with the conservation requirements in \Cref{prop:pact_within_component_characterization}.

\paragraph{Representative selection does not generally recover conserved support.}
Retaining one designated source per component under the same outer-fold threshold procedure gives a pre-admission ncsAURC of 0.864, as against 0.629 for PACT fusion (difference 0.235, 95\% CI $[0.229,0.240]$). After admission, the corresponding values are 0.167 and 0.086. At the target coverage of 0.13, the representative rule yields 246 wrong admissions among 4,051 retained decisions, while no wrong admission is observed for PACT among 4,043 retained decisions.

The two rules may assign different evidence-mass scores even where they predict the same contract, since $1-u$ depends on the retained budget. A representative recovers the meet exactly only when its evidence vector equals the component minimum. Appendix~\ref{sec:additional_experimental_details} reports the frequencies with which scores change.

\paragraph{Provenance misspecification has asymmetric downstream consequences.}
With the numerical evidence fixed, false refinement, merging, recoverable edge deletion, and incorrect distinct-parent assignments alter the supplied counting structure. At a misspecification rate of 50\%, false refinement and incorrect distinct-parent assignments yield $R_{\mathrm{all}}=4.92\%$, whereas merging and recoverable deletion yield 1.06\%, as against 8.75\% for the provenance-unaware singleton reference. Here $R_{\mathrm{all}}$ denotes wrong admissions divided by all evaluation instances. These interventions follow the predicted directions of the budget, although budget monotonicity does not determine predictive or admission ordering. Appendix~\ref{sec:additional_experimental_details} reports the remaining misspecification grid.

False refinement becomes more consequential as output multiplicity increases. Copies inserted within a component and an all-source merge preserve their respective budgets through $m=32$, yet the alternative counting structures produce different decision metrics. Under singleton counting, the aggregate budget at $m=32$ reaches $10.73\times$ its value at $m=1$. Over the support $[0.10,0.39]$, $\Delta_{\mathrm{split}}$ increases from 0.024 (95\% CI $[0.022,0.027]$) at $m=1$ to 0.368 ($[0.353,0.380]$) at $m=8$ and 0.434 ($[0.421,0.447]$) at $m=32$. All six pointwise intervals exclude zero (\Cref{fig:topology_multiplicity_stress}).

The composition of the components also matters when their number is fixed. Grouping language with geometry while risk is left separate raises ncsAURC$_{[0.10,0.39]}$ by 0.199 relative to the reference provenance partition (95\% CI $[0.186,0.211]$). Grouping language with risk while geometry is left separate raises it by 0.201 ($[0.185,0.219]$).

\begin{figure}[pos=t]
\centering
\includegraphics[width=0.9\textwidth]{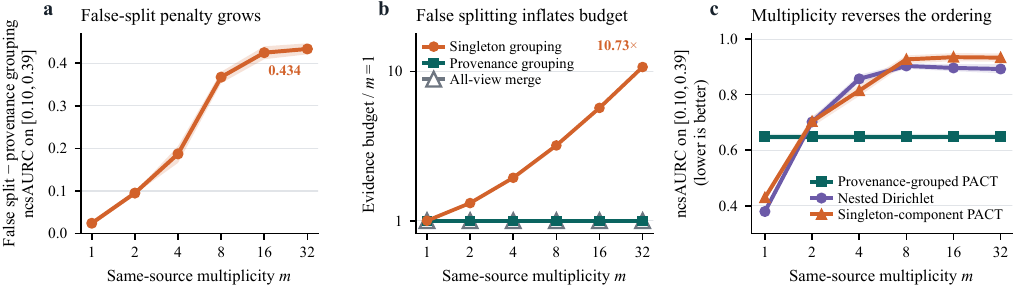}
\caption{False refinement increasingly inflates the evidence budget and increases ncsAURC under the transferred ordering score as multiplicity grows. (a) $\Delta_{\mathrm{split}}=\mathrm{ncsAURC}_{\mathrm{false\ split}}-\mathrm{ncsAURC}_{\mathrm{provenance}}$. (b) Budget relative to $m=1$. (c) ncsAURC under the score estimated at $m=1$. The singleton grouping in panel (b) is the singleton-component PACT rule in panel (c). Bands in panels (a, c) are 95\% paired scene-bootstrap intervals. Panel (b) is deterministic. All panels precede component corroboration.}
\label{fig:topology_multiplicity_stress}
\end{figure}

The multiplicity-transfer comparison uses an ordering score fitted at the baseline multiplicity. Nested Dirichlet has the lower ncsAURC at $m=1$, whereas PACT has the lower value from $m=2$ onward. The PACT-minus-nested difference reaches $-0.254$ at $m=8$ (\Cref{fig:topology_multiplicity_stress}). Removing the observed-source count from the score delays the crossover to $m=4$ without eliminating it. Neither score is refitted after duplication, so the crossover reflects a multiplicity-induced shift under a transferred ordering rule.

% ============================================================
% SECTION 5.2: Selective risk before and after admission
% ============================================================
\subsection{Selective risk before and after admission}
\label{subsec:results_selective_fusion}

\Cref{fig:risk_coverage_comparison} compares separately thresholded pre- and post-admission rules, not a common retained set. Fitting and scene bootstrapping follow Section~\ref{sec:bench_construction}.

\begin{figure}[pos=t]
\centering
\includegraphics[width=0.9\textwidth]{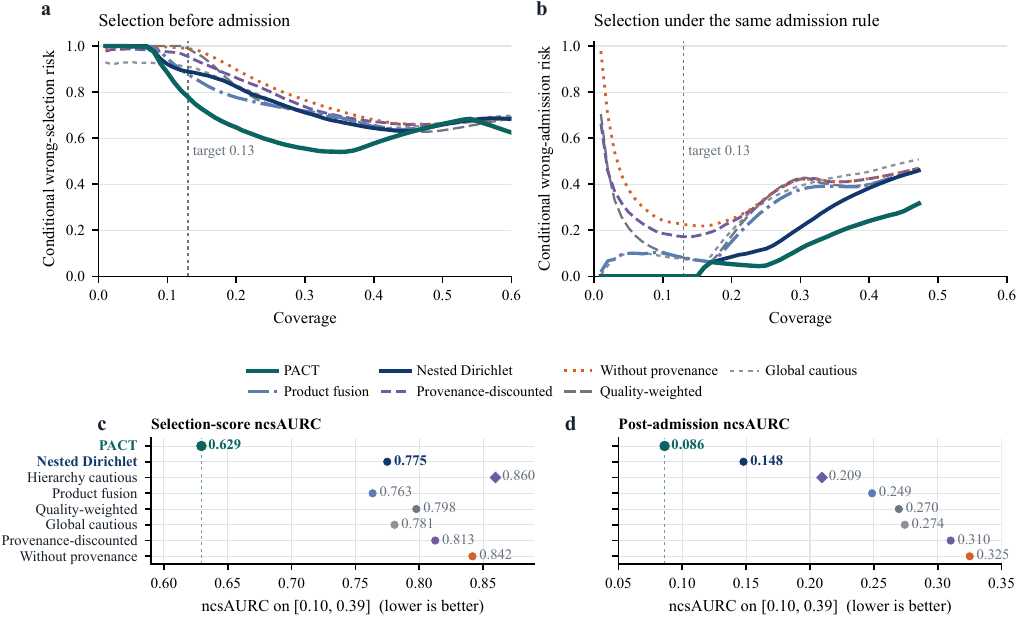}
\caption{Selective risk before and after typed admission. Compared rules are PACT, product fusion, quality-weighted fusion, nested Dirichlet, global cautious, hierarchy-matched cautious, and pooling with and without provenance discounting. Panels (a, b) show conditional risk, while panels (c, d) report ncsAURC$_{[0.10,0.39]}$. Dashed lines mark target coverage 0.13. Panel (c) compares fusion--score combinations before admission, and panel (d) compares the corresponding complete decision rules. Hierarchy-matched cautious appears only in panels (c, d) because its risk--coverage curve does not span the full range shown in panels (a, b).}
\label{fig:risk_coverage_comparison}
\end{figure}

% ============================================================
% SECTION 5.2.1: Fusion methods and selection scores before admission
% ============================================================
\subsubsection{Method ranking depends on the fusion--score pair}

Before admission, the leading fusion method depends strongly on the ordering score. PACT fusion attains the lowest ncsAURC$_{[0.10,0.39]}$ among the native method--score combinations. \Cref{tab:headline_fusion_only_csaurc} reports the quality of the posterior separately from selective ordering under native, posterior-derived, and pooled correctness scores.

\begin{table}[pos=t]
\centering
\caption{Prediction quality and score-dependent selective ordering before admission.}
\label{tab:headline_fusion_only_csaurc}
\begin{threeparttable}
\IFTableSetup
\begingroup
\footnotesize
\renewcommand{\arraystretch}{0.98}
\setlength{\tabcolsep}{2.6pt}
\setlength{\PACTPrimaryTableWd}{0.88\linewidth}
\begin{tabular*}{\PACTPrimaryTableWd}{@{}l@{\extracolsep{\fill}}*{5}{S[table-format=1.4]}@{}}
\toprule
\multicolumn{6}{@{}l}{\textbf{Panel A.} Prediction and probability quality (all instances)} \\
Fusion operator & {Acc.$\uparrow$} & {Macro-F1$\uparrow$} & {NLL$\downarrow$} & {Brier$\downarrow$} & {ECE$_{10}\downarrow$} \\
\midrule
Product & 0.3020 & 0.2996 & 2.7189 & 0.9903 & 0.2778 \\
Nested Dirichlet & 0.3011 & 0.2987 & 1.8146 & 0.8788 & \bfseries 0.0944 \\
Global cautious & 0.2787 & 0.2767 & 2.0974 & 0.9574 & 0.2391 \\
\textbf{PACT fusion} & \bfseries 0.3689 & \bfseries 0.3684 & \bfseries 1.6735 & \bfseries 0.8436 & 0.2454 \\
\midrule
\multicolumn{6}{@{}l}{\textbf{Panel B.} Pre-admission ncsAURC$_{[0.10,0.39]}$ ($\downarrow$) by selection score} \\
\multirow{2}{*}{Fusion operator} & {\multirow{2}{*}{Native}} & \multicolumn{3}{c}{Posterior-derived} & {\multirow{2}{*}{\makecell[c]{Pooled\\correctness$^{a}$}}} \\
\cmidrule(lr){3-5}
& & {Peak} & {Margin} & {Inv. entropy} & \\
\midrule
Product & 0.7635 & 0.7635 & 0.7675 & \bfseries 0.7628 & 0.5547 \\
Nested Dirichlet & 0.7749 & \bfseries 0.7547 & \bfseries 0.7649 & 0.7961 & 0.3796 \\
Global cautious & 0.7806 & 0.7806 & 0.7763 & 0.7904 & 0.4984 \\
Provenance-discounted pooling & 0.8125 & 0.8125 & 0.8065 & 0.8256 & \bfseries 0.3711 \\
Hierarchy-matched cautious & 0.8597 & 0.7870 & 0.7840 & 0.7943 & 0.4659 \\
\textbf{PACT fusion} & \bfseries 0.6294 & 0.9172 & 0.9234 & 0.9158 & 0.6489 \\
\bottomrule
\end{tabular*}
\endgroup
\begin{tablenotes}[flushleft]\footnotesize
\item[] \textit{Note.} Panel A pools 13 source conditions (12 degraded). Acc. denotes accuracy, and Macro-F1 is the classwise average F1 score. NLL denotes negative log-likelihood, Brier denotes the Brier score, and ECE$_{10}$ denotes expected calibration error with ten equal-width bins. Peak, margin, and inverse entropy are $\max_k\pi_k$, $\pi_{(1)}-\pi_{(2)}$, and $1-H(\boldsymbol{\pi})/\log K$, respectively. Bold marks the column best from unrounded results.
\item[\textsuperscript{a}] Fitted on outer-training outputs from product, nested Dirichlet, global cautious, and PACT, then transferred without refitting to the two remaining comparators. Thresholds are method-specific and fitted on outer-training data.
\end{tablenotes}
\end{threeparttable}
\end{table}

Panel B shows that the choice of score can reverse the ranking of the fusion rules. PACT attains the lowest ncsAURC under its native $1-u$ score and the highest under posterior peak, top-two margin, inverse normalized entropy, and pooled correctness. For the same 28,800 base-eligible candidates over the common support $[0.10,0.39]$, the corresponding random-ordering reference is 0.6170. Peak, margin, and inverse entropy exceed that reference by approximately 0.300, 0.306, and 0.299, respectively, whereas pooled correctness exceeds it by approximately 0.032. The native value of 0.6294 for PACT also remains above this random-ordering reference, and product fusion likewise remains above its own method-specific random reference. The native score $1-u$ is monotone in the conserved budget $B_{\Pi}$. The favorable cross-method result applies to PACT together with its native score and does not establish a standalone benefit of $1-u$ for correctness ordering.

Under their native scores, product fusion attains the lowest ncsAURC among the non-PACT operators shown in \Cref{fig:risk_coverage_comparison}. PACT fusion lowers ncsAURC by 0.134 relative to product fusion, with a PACT-minus-product 95\% CI of $[-0.140,-0.128]$ and a negative difference in all 48 scenes. Over the common support $[0.10,0.35]$, the contrast remains negative across all 50 grouped split assignments. Language is the strongest single source on the training folds, yet its ncsAURC is 0.874 and remains unchanged under base eligibility, as against 0.629 for PACT fusion.

The gap of 0.134 between product fusion and PACT also admits a descriptive algebraic decomposition relative to the random-ordering references of the two methods. Of that gap, 0.073 (54.1\%) reflects the difference in candidate-error baselines under score-independent ordering. The remainder compares excess ncsAURC above those baselines for the observed method--score combinations. This decomposition does not identify separate causal contributions of fusion and of scoring. Appendix~\ref{sec:additional_experimental_details} gives the unrounded decomposition.

% ===================================================================
% SECTION 5.2.2: Fusion methods and selection scores under a common admission policy
% ===================================================================
\subsubsection{Identical error counts can conceal selective-risk differences}
\label{subsec:fusion_admission_comparison}

Under the common admission policy, the complete methods differ in their fused candidate, their native selection score, and their fitted threshold. With method-specific scores, PACT attains an ncsAURC of 0.086, as against 0.148 for nested Dirichlet (difference $-0.062$, 95\% CI $[-0.067,-0.056]$). At the target coverage of 0.13, the two methods retain nearly the same number of correct cases, and no wrong admission is observed for either (\Cref{tab:selective_performance}). The methods differ in ncsAURC over the common support $[0.10,0.39]$ but not in observed error counts at that operating point. These finite-sample zeros do not establish zero population risk.

Among instances with three valid sources, posterior-peak ranking narrows the post-admission ncsAURC difference between PACT and nested Dirichlet to $-0.0023$ (PACT-minus-nested difference, 95\% CI $[-0.0028,-0.0018]$). Once the constructed adversarial-consensus condition is removed, the two methods have equal reported ncsAURC under posterior-peak ranking. Product fusion has lower ncsAURC than both in each posterior-peak comparison. Use of the same functional does not fix numerical scores or rankings across methods. Appendix~\ref{sec:additional_experimental_details} reports the values and the subset sizes.

\begin{table}[pos=t]
\centering
\caption{Complete-rule performance and partition--corroboration sensitivity.}
\label{tab:selective_performance}
\begin{threeparttable}
\IFTableSetup
\begingroup
\footnotesize
\renewcommand{\arraystretch}{0.98}
\setlength{\tabcolsep}{2.6pt}
\setlength{\PACTPrimaryTableWd}{0.92\linewidth}
\begin{tabular*}{\PACTPrimaryTableWd}{@{}l@{\extracolsep{\fill}}*{4}{S[table-format=1.4]}@{}}
\toprule
\multicolumn{5}{@{}l}{\textbf{Panel A.} Complete rules with method-specific scores} \\
Method & {ncsAURC$\downarrow$} & {Coverage} & {$R_{\mathrm{all}}\downarrow$} & {$C_{\mathrm{all}}\uparrow$} \\
\midrule
Product & 0.2486 & 0.1301 & 0.0106 & 0.1195 \\
Nested Dirichlet & 0.1479 & 0.1302 & \bfseries 0.0000 & \bfseries 0.1302 \\
Hierarchy-matched cautious & 0.2094 & 0.1300 & 0.0074 & 0.1226 \\
Provenance-discounted pooling & 0.3101 & 0.1298 & 0.0224 & 0.1075 \\
\textbf{PACT} & \bfseries 0.0861 & 0.1296 & \bfseries 0.0000 & 0.1296 \\
\midrule
\multicolumn{5}{@{}l}{\textbf{Panel B.} Matched partition $\times$ corroboration under common $1-u$} \\
Fusion partition & \multicolumn{1}{c}{Corr.}
& {\makecell[c]{ncsAURC\\all instances}}
& \multicolumn{2}{c}{\makecell[c]{ncsAURC\\sources present}} \\
\midrule
\multirow{2}{*}{Singleton}
& \multicolumn{1}{c}{\xmark} & 0.4135 & \multicolumn{2}{c}{0.5214} \\
& \multicolumn{1}{c}{\cmark} & 0.1521 & \multicolumn{2}{c}{0.0426} \\
\multirow{2}{*}{Provenance}
& \multicolumn{1}{c}{\xmark} & 0.3891 & \multicolumn{2}{c}{0.5189} \\
& \multicolumn{1}{c}{\cmark} & \bfseries 0.0861 & \multicolumn{2}{c}{\bfseries 0.0151} \\
\midrule
\multicolumn{5}{@{}l}{\textbf{Panel C.} Sensitivity after adversarial-consensus removal} \\
& {Singleton} & {Provenance}
& \multicolumn{2}{c}{$\Delta$ partition $-$ singleton (95\% CI)} \\
\midrule
ncsAURC & 0.125 & 0.069
& \multicolumn{2}{c}{$-0.056\;[-0.061,-0.051]$} \\
\bottomrule
\end{tabular*}
\endgroup

\begin{tablenotes}[flushleft]\footnotesize
\item[] \textit{Note.} All panels use common support $[0.10,0.39]$. Panel A uses method-specific outer-fold thresholds targeting coverage 0.13, transferred to held-out scenes. $R_{\mathrm{all}}$ and $C_{\mathrm{all}}$ are wrong and correct admissions divided by all instances. A value of 0.0000 denotes no observed wrong admissions. Panels B--C use a common $1-u$ functional, which does not fix numerical scores or retained ordering. In Panel B, sources present comprises 19,200 three-valid-source instances and 2,400 instances with one present but invalid source (21,600 total). Panel C removes the 2,400 adversarial-consensus cases. Corroboration then leaves ncsAURC unchanged. Bold marks descriptive best outcomes from unrounded values. Coverage is not ranked.
\end{tablenotes}

\end{threeparttable}
\end{table}

The matched $2\times2$ comparison of complete decision rules shows that the partition contrast depends on whether component corroboration is applied. Relative to singleton aggregation, provenance partitioning changes ncsAURC by $-0.024$ (95\% CI $[-0.027,-0.022]$) without corroboration and by $-0.066$ ($[-0.072,-0.060]$) with corroboration, which yields an interaction of $-0.042$ ($[-0.045,-0.038]$). Component corroboration targets the constructed adversarial-consensus condition directly, and the response of the policy to that condition is examined in \Cref{subsec:results_verifier_source_support}.

Once the 2,400 adversarial-consensus cases are removed, toggling corroboration no longer changes ncsAURC on the remaining 28,800 evaluations. Aggregation over the provenance partition nevertheless lowers ncsAURC from 0.125 to 0.069 relative to singleton aggregation under the shared $1-u$ score functional (difference $-0.056$, 95\% CI $[-0.061,-0.051]$). The remaining contrast between the two aggregation rules reflects changes in candidate formation and in evidence-mass ordering, which are distinct from the targeted corroboration response.

% ===================================================================================
% SECTION 5.2.3: Admission responses conditional on fusion outputs and scores
% ===================================================================================
\subsubsection{A designed counterexample to agreement as corroboration}
\label{subsec:results_verifier_source_support}

In the 2,400 constructed adversarial-consensus cases, language, geometry, and risk give the same high-confidence support to an incorrect contract while satisfying the probability, quality, and conflict thresholds. Under the reference partition, geometry and risk share a scene-context parent, so the three agreeing outputs span only two complete provenance components. The corroboration check therefore finds $n_{\mathrm{comp}}^{\mathrm{complete}}(k^*;x)=2<\nu=3$ and requests confirmation. Disabling the check admits all 2,400 incorrect cases and gives $R_{\mathrm{all}}=7.69\%$. This stress test characterizes the specified response of the policy rather than the prevalence of that failure mode outside the benchmark.

Reassigning the same numerical opinions to three complete components satisfies the component requirement and admits the candidate. Only the supplied counting structure changes, while the numerical agreement remains the same.

The cross-role sensitivity rule also admits these 2,400 cases, notwithstanding that their support spans only two complete provenance components. Appendix~\ref{sec:additional_experimental_details} reports the full admission-check ablation, the alternative counting rules, and the illustrative decision-cost comparison.

% ============================================================
% SECTION 5.3: Source availability, score dependence, and calibration
% ============================================================
\subsection{Missing evidence, score dependence, and calibration}
\label{subsec:results_robustness_generalization}

% ==================================================
% SECTION 5.3.1: Expected-source availability
% ==================================================
\subsubsection{Missing-source treatment changes evidence accumulation and selective risk}

The treatment of missing evidence changes the fused prediction even where the observed source outputs are unchanged. Under the primary retained-zero convention, every declared source remains in its provenance component, so an incomplete component contributes no conserved support. Under this convention, ncsAURC is 0.629 before admission and 0.086 afterward. Excluding the zero vector from the component meet, removing the unavailable source, or isolating it instead gives 0.775 before admission and 0.188 afterward in each case. The latter post-admission value also exceeds the 0.148 obtained by nested Dirichlet under its own missing-source convention. \Cref{subsec:fusion_admission_comparison} gives the corresponding three-valid-source comparison.

The effect of zero evidence depends on the aggregation rule. Under the stated nested-Dirichlet composition, zero evidence for source $i$ gives $\mathbf b_i=\mathbf 0$ and $u_i=1$, a vacuous opinion that leaves the combined opinion unchanged. In the within-component meet of PACT, the same zero evidence vector sets the retained budget of that component to zero.

The increase in ncsAURC for PACT under these alternative missing-source conventions exceeds the pre-admission margin of 0.134 over product fusion. The three alternatives coincide here because unavailable sources do not bridge observed components that are otherwise disconnected. In this sensitivity analysis on source unavailability, excluding the zero changes 1,670 predictions among the 9,600 instances with an unavailable source, with a mean posterior $\ell_1$ drift of 0.097. No change of prediction occurs among the 21,600 instances without source unavailability. This set includes present-but-invalid instances and therefore differs from the three-valid-source subset.

Stratifying by availability and validity reveals a reversal in the ordering of the methods. Among the 19,200 instances with three valid sources, pre-admission ncsAURC is 0.617 for PACT and 0.725 for nested Dirichlet. Among the 12,000 instances with an unavailable or invalid source, PACT has the higher value at 0.931 against 0.867 for nested Dirichlet. The same held-out scores are ranked within each stratum, so the reversal does not arise from refitting of the score or the threshold. Results for product fusion and for all-valid eligibility are reported in Appendix~\ref{sec:additional_experimental_details}.

Score dependence persists on the three-valid-source subset. With the pooled correctness scorer, PACT has the highest ncsAURC among the four compared rules both before and after admission. Before admission, the value of 0.803 for PACT contrasts with 0.384 for the best-performing global cautious rule. After admission, PACT gives 0.367 against 0.048 for the best-performing nested Dirichlet rule (\Cref{tab:pooled_score_valid_sources}). This reverses the favorable pre-admission ordering obtained for PACT under $1-u$.

Equal source scales define the primary benchmark specification. Source-specific rescaling selected on the training folds preserves the ordering between PACT and nested Dirichlet before and after admission, although equal source scales yield the lower ncsAURC for PACT in both comparisons. Appendix~\ref{sec:additional_experimental_details} reports the selected scales and the sensitivity results.

% ===========================================================
% SECTION 5.3.2: Selective ordering and probability calibration
% ===========================================================
\subsubsection{Lower NLL and Brier scores coexist with higher ECE}
\label{subsec:selective_calibration}

Lower NLL and Brier scores coexist with higher top-confidence calibration error in the comparison between PACT and nested Dirichlet. PACT also attains the lower ncsAURC under the native selection scores of the two methods. The PACT-minus-nested differences are $-0.141$ for NLL (95\% CI $[-0.149,-0.133]$), $-0.035$ for Brier ($[-0.036,-0.034]$), and $+0.151$ for ECE$_{10}$ ($[0.141,0.158]$).

An empirical-prior reference estimated from the training portions of the outer folds reaches an NLL of 1.609 and a Brier score of 0.800, both lower than the corresponding values of 1.674 and 0.844 for PACT.

\subsubsection{Aggregate ordering masks condition-specific reversals}
\label{subsec:conditionwise_joint_holdout}

In the joint scene-and-condition holdout, method-specific score-cutoff fitting excludes both the outer test-scene fold and the held source condition. The concentration schedule is transferred from the main study without reselection for each held condition. \Cref{tab:joint_scene_configuration_holdout} separates the aggregate comparison from the condition-specific contrasts for which both methods attain the common support $[0.10,0.35]$.

\begin{table}[pos=t]
\centering
\caption{Joint scene-and-condition holdout on common support $[0.10,0.35]$.}
\label{tab:joint_scene_configuration_holdout}
\begin{threeparttable}
\IFTableSetup
\begingroup
\footnotesize
\renewcommand{\arraystretch}{0.98}
\setlength{\tabcolsep}{2.6pt}
\begin{tabularx}{0.94\linewidth}{@{}
>{\raggedright\arraybackslash}X
S[table-format=-1.4]
S[table-format=-1.4]
S[table-format=1.4]
S[table-format=1.4]@{}}
\toprule
\multicolumn{5}{@{}l}{\textbf{Panel A.} Complete rules under the common admission policy} \\
Method & {ncsAURC$\downarrow$} & {Coverage} & {$R_{\mathrm{all}}\downarrow$} & {$C_{\mathrm{all}}\uparrow$} \\
\midrule
\textbf{PACT} & \bfseries 0.0617 & 0.1538 & \bfseries 0.0000 & \bfseries 0.1538 \\
Nested Dirichlet & 0.1168 & 0.1538 & \bfseries 0.0000 & \bfseries 0.1538 \\
Hierarchy-matched cautious & 0.1775 & 0.1609 & 0.0071 & \bfseries 0.1538 \\
Product & 0.2247 & 0.1644 & 0.0106 & 0.1538 \\
Provenance-discounted pooling & 0.2828 & 0.1708 & 0.0225 & 0.1483 \\
\addlinespace[1pt]
\multicolumn{2}{@{}l}{Paired PACT $-$ nested: $\Delta$ ncsAURC}
& -0.0551
& \multicolumn{2}{c@{}}{$95\%$ CI $[-0.0595,-0.0507]$} \\
\midrule
\multicolumn{5}{@{}l}{\textbf{Panel B.} Condition-wise PACT--nested Dirichlet contrasts} \\
Held-out source condition
& {$\Delta$ ncsAURC}
& \multicolumn{2}{c}{$95\%$ CI}
& {Lower} \\
\midrule
Missing geometry under low-quality language
& -0.4108
& \multicolumn{2}{c}{$[-0.4575,-0.3734]$}
& \multicolumn{1}{l}{PACT} \\
Low-quality partial disagreement
& -0.0554
& \multicolumn{2}{c}{$[-0.0600,-0.0511]$}
& \multicolumn{1}{l}{PACT} \\
Missing language under noisy scene context
& 0.0252
& \multicolumn{2}{c}{$[0.0092,0.0411]$}
& \multicolumn{1}{l}{Nested Dirichlet} \\
High-confidence contradiction
& 0.1160
& \multicolumn{2}{c}{$[0.0920,0.1423]$}
& \multicolumn{1}{l}{Nested Dirichlet} \\
\bottomrule
\end{tabularx}
\endgroup

\begin{tablenotes}[flushleft]\footnotesize
\item[] \textit{Note.} Panel A aggregates 13 held-condition test subsets. Panel B reports the four conditions for which both methods attain common support. The main-study concentration schedule is transferred without reselection, and method-specific thresholds exclude the held condition and test scenes. $R_{\mathrm{all}}=0.0000$ denotes no observed wrong admissions. In Panel B, $\Delta$ is PACT minus nested Dirichlet. Negative values favor PACT. Intervals use 2,000 paired scene bootstraps. Bold follows unrounded values.
\end{tablenotes}

\end{threeparttable}
\end{table}

The four held-out conditions with common support split the ranking. PACT has the lower ncsAURC for missing geometry under low-quality language and for low-quality partial disagreement, whereas nested Dirichlet has the lower value for missing language under noisy scene context and for high-confidence contradiction. Thresholds are transferred without refitting to the individual test subsets, and coverage consequently shifts across conditions. At the reported aggregate operating point, no wrong admission is observed for either method and both retain the same number of correct cases, while their ncsAURC values differ over the common coverage interval.

Across all 256 Sobol settings, the aggregate post-admission ordering remains unchanged. This parameter sensitivity does not remove the condition-specific reversals observed under the joint holdout. Appendix~\ref{sec:additional_experimental_details} reports the corresponding ranges and further sensitivity results.

% =======================================================================
% SECTION 6: Structural Transfer to Learned Predictions and HRC Admission
% =======================================================================
\section{Structural Transfer to Learned Predictions and HRC Admission}
\label{sec:external_evidence}

The studies on learned outputs evaluate changes in the evidence budget separately from accuracy, selective ordering, and admission.

% ====================================================
% SECTION 6.1: Multi-view provenance interventions
% ====================================================
\subsection{Structural transfer on learned multi-view predictions}
\label{subsec:public_feature_view_transfer}

On the learned outputs of TMC and RCML, exact copies assigned within their reference component preserve the budget, false refinement expands it, and all-view merging contracts it. Across the HandWritten and PIE dataset--model pairs, PACT fusion remains invariant to exact within-component copies through multiplicity eight, and every paired evidence-budget interval for false refinement and for all-view merging excludes zero. HandWritten and Scene15 also serve to test accuracy and ncsAURC, whereas PIE is used only for the budget analysis. Appendix~\ref{sec:additional_experimental_details} reports the evidence-budget result on PIE.

On the Scene15--TMC pair, false refinement expands the mean evidence budget to $3.32\times$ its original-view value, while all-view merging retains only $0.053\times$. Both assignments reduce accuracy and increase ncsAURC relative to original-view grouping. Since a different common-support interval is used for HandWritten, ncsAURC is interpreted within each dataset rather than compared across datasets. Appendix~\ref{subsec:scene15_tmc_details} reports the replication accuracy, the paired intervals, and the normalization sensitivity.

\begin{table}[pos=t]
\centering
\caption{Effects of provenance reassignment on HandWritten/Mfeat at multiplicity $m=8$.}
\label{tab:public_outcome_closure}
\begin{threeparttable}
\IFTableSetup
\begingroup
\footnotesize
\renewcommand{\arraystretch}{1.02}
\setlength{\tabcolsep}{2.4pt}
\begin{tabularx}{\linewidth}{@{}
>{\raggedright\arraybackslash}p{2.85cm}
>{\centering\arraybackslash}X
>{\centering\arraybackslash}X
>{\centering\arraybackslash}X@{}}
\toprule
Dataset--model
& $\Delta$ ncsAURC
& $\Delta$ posterior $\ell_1$
& $\Delta$ accuracy (pp) \\
\midrule
\multicolumn{4}{@{}l}{\textit{False refinement} \; (reference: within-component copies)} \\
HandWritten--TMC
& $+0.0049\;[0.0004,0.0108]$
& $+0.246\;[0.243,0.249]$
& $-0.37\;[-0.83,+0.03]$ \\
HandWritten--RCML
& $+0.0024\;[0.0008,0.0044]$
& $+0.214\;[0.209,0.220]$
& $-1.25\;[-1.91,-0.72]$ \\
\addlinespace[2pt]
\multicolumn{4}{@{}l}{\textit{All-view merge} \; (reference: original views)} \\
HandWritten--TMC
& $+0.4894\;[0.4643,0.5155]$
& $+0.900\;[0.876,0.924]$
& $-62.60\;[-65.80,-59.50]$ \\
HandWritten--RCML
& $+0.0182\;[0.0091,0.0286]$
& $+1.137\;[1.117,1.155]$
& $-7.05\;[-9.35,-5.00]$ \\
\bottomrule
\end{tabularx}
\endgroup
\begin{tablenotes}[flushleft]\footnotesize
\item[] \textit{Note.} The abbreviation pp denotes percentage points. Brackets give two-sided 95\% intervals from 2,000 class-stratified, record-paired bootstrap draws. Each contrast is intervention minus the stated reference. Posterior drift is the mean $\ell_1$ distance from the original-view posterior. Its $\Delta$ compares this drift across arms. ncsAURC uses common support $[0.10,0.90]$. Positive $\Delta$ indicates higher conditional risk. Reassignment may change the fused candidate, so the contrasts do not isolate numerical ordering.
\end{tablenotes}

\end{threeparttable}
\end{table}

Exact within-component copies leave accuracy on HandWritten at 98.0\% for TMC and 98.5\% for RCML. False refinement shifts the posterior by 0.21--0.25 in $\ell_1$ and increases ncsAURC for both models. The accuracy-change interval for TMC includes zero, whereas RCML loses 1.25 percentage points (95\% CI $[-1.91,-0.72]$, \Cref{tab:public_outcome_closure}). The structural response of the budget therefore appears even where the corresponding effect on accuracy is small or unresolved.

All-view merging reduces accuracy on HandWritten by 62.60 percentage points for TMC and by 7.05 points for RCML, notwithstanding the larger posterior drift observed for RCML. Across five realizations of the Scene15--TMC pair, merging reduces accuracy by 8.49 percentage points, as against 4.62 under false refinement. Predictive loss thus depends on the dataset and the model rather than on posterior displacement alone.

False refinement may improve probability scores while it worsens selective ordering. On HandWritten, false refinement relative to the reference grouping lowers both the macro-averaged NLL and the Brier score while increasing ncsAURC in four of the 12 replicated-source choices, six for each model. Some of these assignments also reduce ECE$_{10}$.

% ===========================================================
% SECTION 6.1.1: Exhaustive coarsening of original feature views
% ===========================================================
\subsubsection{Budget contraction and nonmonotone selective risk}
\label{subsec:partition_coarsening_surface}

\begin{figure}[pos=t]
\centering
\includegraphics[width=0.9\linewidth]{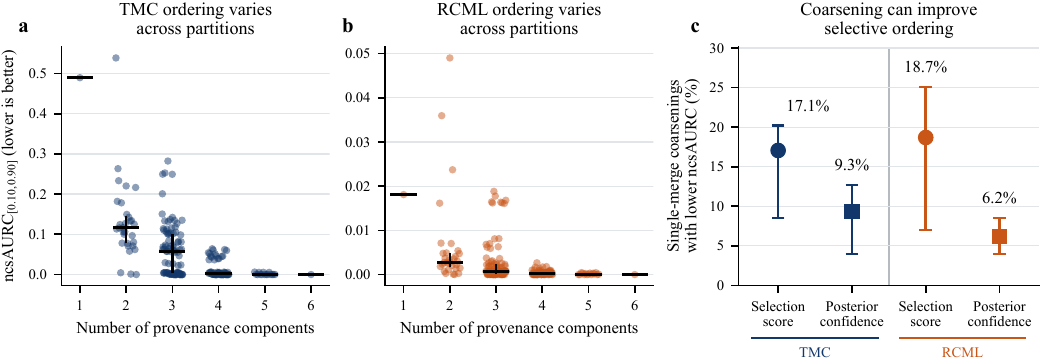}
\caption{Evidence-budget contraction does not determine selective risk under exhaustive coarsening of six-view HandWritten outputs. Panels (a,b) show selection-score ncsAURC across 203 partitions by component count for TMC and RCML, respectively, using different vertical scales. Panel (c) reports the percentage of the 856 single-merge relations for which coarsening lowers ncsAURC under the selection score and posterior confidence. Whiskers show 95\% class-stratified record-bootstrap intervals.}
\label{fig:partition_coarsening_surface}
\end{figure}

Across the 203 partitions of the six HandWritten views, all 856 single-merge relations satisfy the predicted budget monotonicity, yet the ncsAURC response is not monotone. Under the selection score, coarsening lowers ncsAURC for 17.1\% of these merges with TMC and for 18.7\% with RCML. Under posterior-confidence ranking, the corresponding proportions are 9.3\% and 6.2\% (\Cref{fig:partition_coarsening_surface}).

\FloatBarrier

% ============================================================
% SECTION 6.2: Camera-acquisition and model-family grouping
% ============================================================
\subsection{Camera-acquisition and model-family grouping}
\label{subsec:habit_checks}

The HABIT analyses~\citep{Song2026HABIT} use 1,128 paired pre-release events from 696 episodes across six tasks, together with the temporal windows defined in \Cref{subsec:external_transfer_design}. Prompt outputs derived from the same camera image form one acquisition component in the camera-grouping analysis. This observed identity of the acquisition defines the reference grouping at the granularity evaluated here and does not represent a structure of statistical independence. Model-family labels are evaluated separately as a surrogate grouping hypothesis against the empirical association of errors.

% ========================================
% SECTION 6.2.1: Camera-acquisition grouping
% ========================================
\subsubsection{Acquisition grouping changes budgets and candidates}
\label{subsec:native_view_fm_transfer}

At two prompts per camera, acquisition grouping yields lower ncsAURC than per-output counting under the selection score for both checkpoints. At four prompts, the per-output-minus-acquisition contrast remains positive for Qwen3-VL-8B at 0.120 (95\% CI $[0.089,0.150]$), whereas the interval for Qwen3-VL-32B includes zero. For Qwen3-VL-8B, posterior-confidence ranking reverses the relative ordering of acquisition grouping and per-output counting. The consequence of counting for ncsAURC therefore depends on both the checkpoint and the score (\Cref{fig:native_view_fm_transfer}).

The accuracy response also changes sign across checkpoints. Relative to per-output counting, acquisition grouping improves the accuracy of Qwen3-VL-32B by 3.16 percentage points but lowers that of Qwen3-VL-8B by 3.14 points. All-view merging lowers accuracy relative to acquisition grouping for both checkpoints, yet selection-score ncsAURC changes only from 0.226 to 0.233 for the 32B checkpoint and from 0.256 to 0.257 for the 8B checkpoint (\Cref{tab:native_view_fm_transfer}). Accuracy and selective ordering may therefore respond differently to the same counting assignment.

At four prompts, equal-cardinality shuffled assignments retain less evidence than acquisition grouping for both checkpoints. Component size alone therefore does not explain this budget contrast. The posterior-confidence and equal-cardinality-shuffling analyses are post-hoc.

Unlike the rule-specific scores in \Cref{tab:native_view_fm_transfer}, the common-score comparison fixes the acquisition-grouped numerical mass scores across both counting rules. For Qwen3-VL-32B, the corresponding ncsAURC values are 0.226 and 0.247, which gives a per-output-minus-acquisition difference of 0.021 (95\% paired episode-bootstrap CI $[0.0015,0.0449]$). For Qwen3-VL-8B, the values are 0.256 and 0.250, which gives a difference of $-0.006$ (95\% CI $[-0.0134,-0.0004]$). Holding the numerical ranking fixed does not remove the checkpoint-dependent difference in the construction of candidates.

\begin{figure}[pos=t]
\centering
\includegraphics[width=0.9\textwidth]{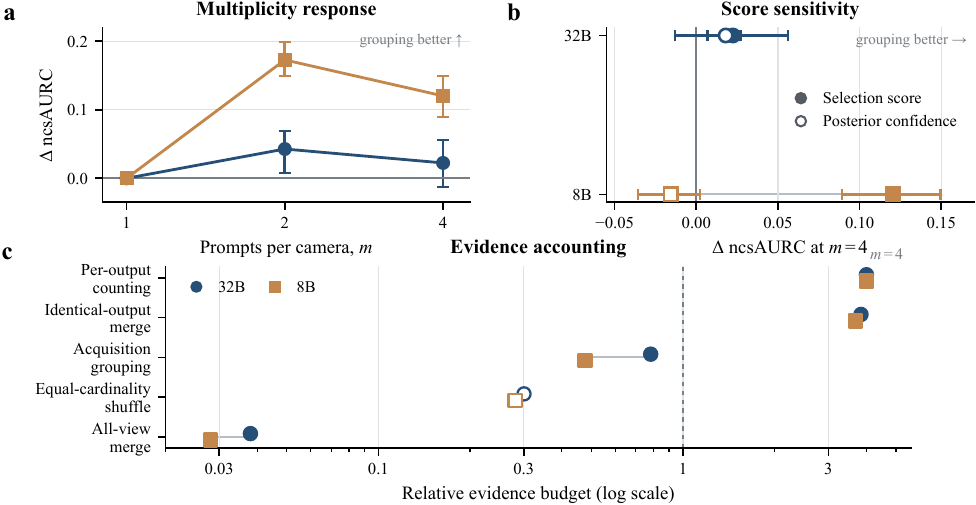}
\caption{Acquisition grouping changes retained evidence, while its downstream consequences depend on checkpoint and score across 1,128 HABIT events. (a) Difference in ncsAURC between per-output counting and acquisition grouping as prompt multiplicity increases. (b) At four prompts, the same contrast under the selection score (filled) and posterior confidence (open). (c) Evidence budgets under alternative counting rules relative to one output per camera. Hollow markers denote equal-cardinality shuffled groupings. Whiskers in panels (a,b) are 95\% paired episode-bootstrap intervals. Positive differences in panels (a,b) favor acquisition grouping.}
\label{fig:native_view_fm_transfer}
\end{figure}

% ========================================
% SECTION 6.2.2: Model-family grouping
% ========================================
\subsubsection{Error association does not uniquely identify model-family grouping}
\label{subsec:checkpoint_family_grouping}

The supplied model-family grouping changes the retained budget, but observed error associations do not single it out among the tested pairings. Appendix~\ref{sec:additional_experimental_details} gives the operator comparison.

\begin{figure}[pos=t]
\centering
\includegraphics[width=0.9\textwidth]{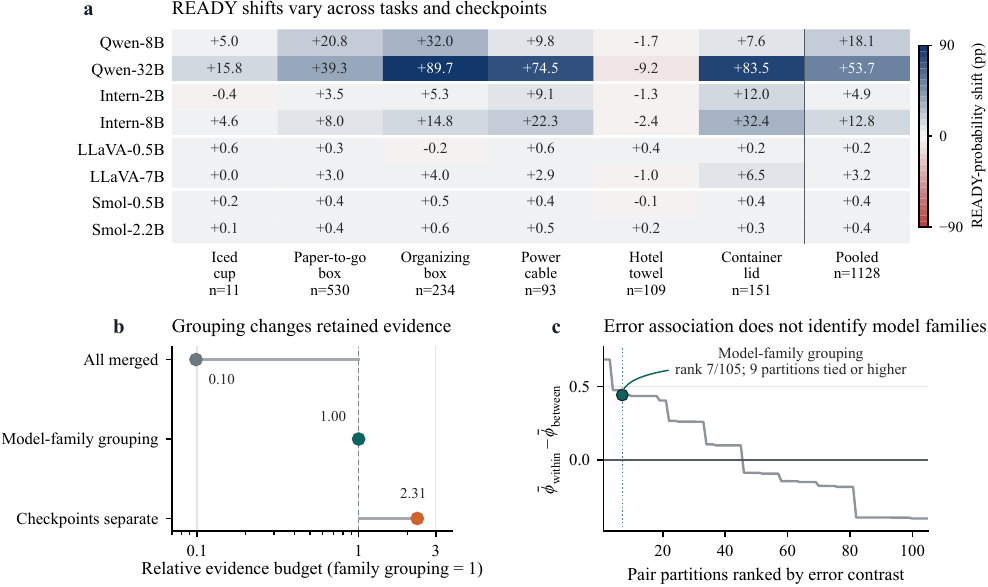}
\caption{Observed error association does not uniquely recover model-family grouping. (a) Early-to-event-proximal change in mean \textsc{ready} probability across eight checkpoints and six HABIT tasks. Here $n$ denotes paired events. (b) Evidence budgets under checkpoint merging and separation relative to model-family grouping. (c) Within-minus-between error association for all 105 pair partitions under pooled event-window weighting. Model-family pairing ranks seventh, with nine pairings tied or higher.}
\label{fig:balanced_fm_panel}
\end{figure}

Within the supplied model-family grouping, exact copies leave the evidence budget and the posterior unchanged. Separating checkpoints increases the mean event-window budget to $2.306\times$ the value obtained under model-family grouping, and accuracy is left unchanged. Merging all checkpoints reduces the mean budget ratio to 0.099 and lowers accuracy from 86.6\% to 76.3\%. Appendix~\ref{sec:additional_experimental_details} reports the corresponding posterior drifts.

Model-family pairing ranks seventh among 105 pairings, with nine tied or higher under pooled event-window weighting (\Cref{fig:balanced_fm_panel}). Alternative weightings and post-hoc sensitivity are reported in Appendix~\ref{sec:additional_experimental_details}.

% ========================================================
% SECTION 6.3: HRC admission with complementary evidence
% ========================================================
\subsection{HRC admission with complementary evidence}
\label{subsec:habit_target_temporal_transfer}

The HRC analyses use 60 evaluation episodes that do not overlap with the development set. No evidence-mass cutoff is applied at admission. Candidate evidence, target identity, and temporal evidence form complementary release requirements, and in the learned event-proximity analysis the candidate evidence and the temporal signal share camera acquisitions. A candidate is termed \emph{reference-consistent} when it agrees with the dataset reference action and \emph{reference-inconsistent} otherwise.

For Panel A of \Cref{tab:habit_target_temporal_summary}, the learned event-proximity model is trained on development episodes drawn from all six task categories and is evaluated on unseen episodes from those same categories. The separate leave-one-task-out comparison, which excludes the evaluated task from fitting, is reported in Appendix~\ref{sec:additional_experimental_details}.

\begin{table}[pos=htbp]
\centering
\caption{Offline HRC admission on 60 evaluation episodes that do not overlap with the development set.}
\label{tab:habit_target_temporal_summary}
\begin{threeparttable}
\IFTableSetup
\begingroup
\footnotesize
\renewcommand{\arraystretch}{1.00}
\setlength{\tabcolsep}{2.7pt}
\begin{tabular*}{0.96\linewidth}{@{}l@{\extracolsep{\fill}}*{4}{S[table-format=2.0]}@{}}
\toprule
\multicolumn{5}{@{}l}{\textbf{Panel A.} Admission requirements with learned event proximity (Qwen3-VL-8B)} \\
Admission evidence
& \multicolumn{3}{c}{Reference-inconsistent admissions}
& {\makecell[c]{Reference-consistent\\admissions}} \\
\cmidrule(lr){2-4}
& {Off-target}
& {\makecell[c]{On-target\\early-window}}
& {Total}
& {} \\
\midrule
Candidate evidence only    & 66 & 25 & 91 & 53 \\
Target identity            & 0  & 25 & 25 & 53 \\
Event proximity            & 53 & 1  & 54 & 51 \\
\textbf{Target + event proximity} & 0 & 1 & 1 & 51 \\
\midrule
\multicolumn{5}{@{}l}{\textbf{Panel B.} Counting rules with target identity and strict positive change} \\
& \multicolumn{2}{c}{Qwen3-VL-32B}
& \multicolumn{2}{c}{Qwen3-VL-8B} \\
\cmidrule(lr){2-3}\cmidrule(lr){4-5}
Counting rule
& {Consistent}
& {Inconsistent}
& {Consistent}
& {Inconsistent} \\
\midrule
\textbf{Camera-grouped PACT} & 47 & 0 & 43 & 0 \\
Per-output counting          & 48 & 0 & 42 & 0 \\
One output per camera        & 45 & 0 & 44 & 0 \\
\bottomrule
\end{tabular*}
\endgroup

\begin{tablenotes}[flushleft]\footnotesize
\item[] \textit{Note.} Reference-consistent admissions are counted among 60 reference-\textsc{ready} episode targets. Reference-inconsistent admissions are counted across 720 queries per checkpoint. Episodes, not queries, are the resampling unit. Panel A uses learned event proximity. Panel B fixes target identity and the strict positive-change temporal criterion. Consistency denotes agreement with the dataset reference action, not task success. Zero denotes no observed reference-inconsistent admission in this finite evaluation, not a population guarantee. The evidence-mass score is not a release threshold.
\end{tablenotes}

\end{threeparttable}
\end{table}

For Qwen3-VL-8B in Panel A, target identity removes all 66 observed off-target admissions without any reduction in the 53 reference-consistent admissions. The learned event-proximity requirement reduces on-target early-window admissions from 25 to one while 51 of those 53 reference-consistent cases are retained. For reference-inconsistent admissions, the joint-minus-candidate-only difference in rate is $-12.50$ percentage points (95\% CI $[-13.75,-11.25]$). The comparison uses episode-stratified, equal-task weighting, with episodes as the resampling unit. Appendix~\ref{sec:additional_experimental_details} gives the counts and the rates.

Panel B instead adopts a strict increase in the fused \textsc{ready} posterior from the early to the proximal window as its temporal requirement. Across its matched counting arms, candidate selection and target identity are fixed, while the counting rule determines the fused temporal posterior on which this requirement operates.

With camera-acquisition grouping, PACT admits 47 of 57 reference-consistent candidates from Qwen3-VL-32B and 43 of 53 reference-consistent candidates from Qwen3-VL-8B. At the reported operating point, camera-grouped PACT, per-output counting, and one-output-per-camera counting each yield no observed reference-inconsistent admission across 720 queries for either Qwen3-VL checkpoint. Their retention ordering nevertheless varies across checkpoints (\Cref{tab:habit_target_temporal_summary}).

Under the binary equal-mass mapping, with four prompt variants per camera and all other admission inputs held fixed, duplicating each prompt output within its camera from multiplicity one to eight leaves the budget, the posterior, the temporal comparison, and all 720 typed responses per checkpoint unchanged.

Per-output counting may instead amplify the evidence budget under duplication. This increase alone does not determine the typed response, since release depends on the temporal comparison rather than on an evidence-mass cutoff (Appendix~\ref{sec:additional_experimental_details}).

Retention loss is concentrated in the Hotel Towel task, where camera-grouped PACT admits none of the ten reference-\textsc{ready} episodes at either checkpoint. The temporal requirement rejects all ten candidates from Qwen3-VL-32B and the nine candidates from Qwen3-VL-8B that pass candidate selection. Appendix~\ref{sec:additional_experimental_details} gives the results for the counting arms. These counts identify where retention is lost but do not establish whether withholding was physically necessary.

% =======================
% SECTION 7: Discussion
% =======================
\section{Discussion}
\label{sec:discussion}

% =================================================================
% SECTION 7.1: Provenance partition as an evidence-counting variable
% =================================================================
\subsection{Evidence countability as a structural fusion variable}
\label{subsec:discussion_if_implications}

Insertion non-amplification prevents an additional output from raising the retained budget of an existing component, and this holds even where the output differs from the existing members and carries predictive value. Predictive ordering depends in addition on the data and on the selection score.

Some reassignments relative to the reference grouping improve proper scores, and in certain cases ECE$_{10}$, while selective performance worsens (\Cref{subsec:public_feature_view_transfer}). The quality of the posterior alone therefore cannot justify a decision to treat replicated outputs as separately countable where the reference grouping places them together.

% ==========================================================================
% SECTION 7.2: Implications for multimodal foundation models and embodied systems
% ==========================================================================
\subsection{Implications for multimodal foundation models and embodied systems}
\label{subsec:discussion_foundation_models}

Repeated inference with foundation models can generate further outputs from the same acquired observations. Self-consistency samples multiple reasoning paths from a single input~\citep{Wang2023SelfConsistency}, and related language models may share correlated errors~\citep{Kim2025CorrelatedErrors}. In the camera studies, prompt outputs obtained from the same image inherit a common acquisition parent. Model-family labels provide a different, surrogate grouping hypothesis, yet the observed associations of errors do not uniquely recover that grouping (\Cref{subsec:checkpoint_family_grouping}). Agreement alone does not identify countability, and the counting relation therefore requires acquisition records or some other justification.

Selective retention, typed admission, and physical safety control address different decisions. Admission requirements supplement predictive selection. Execution still requires physical safeguards.

Temporal requirements likewise determine when evidence can support admission. Absolute event-proximity scores shift across tasks, while within-episode change compares a later observation with its own earlier reference. A relative-change criterion cannot support the earlier decision, since its second observation does not yet exist. Sequential admission must therefore respect both which evidence may be counted separately and when that evidence becomes available.

\paragraph{Consequences of partition coarsening.}
For a fixed source index set, evidence collection, $W$, base eligibility, and threshold $\tau$, \Cref{cor:pact_partition_monotonicity} gives $B_{\Pi_2}\leq B_{\Pi_1}$ and $s_{\Pi_2}\leq s_{\Pi_1}$ under coarsening. The coarser partition therefore retains a subset of the instances retained by the finer partition at that threshold. With the candidate $k^*$ and the observed sources fixed, coarsening also cannot increase the number of complete supporting components, since every complete supporting coarse block contains a complete supporting fine block. These bounds on the budget and on the count place no constraint on conditional risk or on physical safety. Refitting the threshold can break the set inclusion, and a change of candidate can change admission. The score-retention result does not apply to the HRC studies, in which no mass cutoff is used.

% ============================================================
% SECTION 7.3: Scope, limitations, and research directions
% ============================================================
\subsection{Scope, limitations, and research directions}
\label{subsec:limitations}

The HRC study measures agreement with dataset reference actions rather than physical task success. Observing no reference-inconsistent admission at a particular operating point does not imply zero population risk. Retention remains dependent on the task, as illustrated by the Hotel Towel failures in \Cref{subsec:habit_target_temporal_transfer}. Studies involving robot execution and human participants would be needed to assess task success, the necessity of withholding, and the outcomes of physical interaction.

Missing-source conventions change how fusion operators combine the same zero-evidence vector. These conventions, together with the evidence scales, form part of the fusion method.

The experiments use constructed provenance in the benchmark and observed acquisition identity in the camera studies. Where provenance is only partly observed, acquisition records and other information about the derivation of sources may assist in inferring the grouping. Evidence scales require validation for each application, and admission costs need to be elicited and validated for both the release and the withholding of actions.
\Needspace{8\baselineskip}

% =========================================
% SECTION 8: Conclusion
% =========================================
\section{Conclusion}
\label{sec:conclusion}

Evidence countability is made explicit in PACT through a supplied provenance partition. Under singleton fidelity and insertion non-amplification, the coordinatewise meet is the unique pointwise greatest admissible within-component rule. Component budgets are added across units supplied as separately countable, under the commensurability and additivity assumptions in \Cref{assump:commensurate_evidence,assump:separate_component_additivity}. The cumulative accounting rule is then invariant under exact-copy insertion and satisfies partition-coarsening monotonicity.

The experiments indicate why predictive quality alone cannot validate a counting assignment. Some false-refinement assignments relative to the reference grouping improve NLL and the Brier score and, in certain cases, ECE$_{10}$, while selective ordering is worsened. In the offline HRC study, exact within-camera duplication leaves the typed responses unchanged whenever all other admission inputs are held fixed. Predictive and admission outcomes nevertheless depend on the model, the selection score, the missing-source convention, and the release conditions. Repeated computation may refine a candidate without thereby establishing a further evidential origin.

\FloatBarrier
\appendix

% =====================================================
% APPENDIX A: Proofs and Additional Stability Results
% =====================================================

\section{Proofs and Additional Stability Results}
\label{sec:proofs}

% ==================
% APPENDIX A.1: Notation
% ==================
\subsection{Notation}
\Cref{tab:notation} summarizes the symbols used in the main text and in the appendix.

\begin{table}[pos=t]
\centering
\caption{Notation for PACT fusion and typed admission.}
\label{tab:notation}
\IFTableSetup
\begingroup
\footnotesize
\renewcommand{\arraystretch}{1.00}
\setlength{\tabcolsep}{2.4pt}
\begin{tabularx}{\linewidth}{@{}>{\raggedright\arraybackslash}p{2.25cm}>{\raggedright\arraybackslash}X>{\raggedright\arraybackslash}p{2.35cm}>{\raggedright\arraybackslash}X@{}}
\toprule
\multicolumn{2}{c}{Source and evidence} & \multicolumn{2}{c}{Fusion and decision} \\
\cmidrule(lr){1-2}\cmidrule(lr){3-4}
\multicolumn{1}{c}{Symbol} & \multicolumn{1}{c}{Meaning} & \multicolumn{1}{c}{Symbol} & \multicolumn{1}{c}{Meaning} \\
\midrule
$\mathcal K$ & finite set of $K\geq2$ action contracts with a fixed tie-breaking order & $\mathbf b_C$ & component evidence budget, the meet of member evidence vectors \\
$x$ & decision instance & $\mathbf E_{\Pi}$ & cumulative evidence under partition $\Pi$ \\
$J$ & size of the expected source set & $B_{\Pi}$ & total budget $\lVert\mathbf E_{\Pi}\rVert_1$ \\
$\xi_i$, $\Phi_i$ & source representation and its evidence adapter & $\boldsymbol\pi$ & projected categorical posterior \\
$\mathbf p_i$ & categorical opinion or zero vector & $\widehat k$ & predicted contract \\
$q_i$, $d_i$ & source quality and conflict scores & $u$, $s$ & vacuity and selection score $s=1-u$ \\
$z_i$, $v_i$ & unavailability and required-field numerical-validity indicators & $F_{\Pi}$ & PACT fusion map \\
$P_i(x)$ & provenance parent set & $\tau$ & score-selection threshold \\
$\mathcal I_{\mathrm{obs}}$, $\mathcal I_k$ & observed sources and sources supporting $k$ & $\nu$, $\mathcal U_3$ & required number of complete supporting components and the indicator that all three expected sources are observed and support the selected candidate \\
$G_{\mathrm{prov}}$, $\mathcal E_P$, $\Pi_P$ & parent-overlap graph, its edge set, and induced component partition & $\zeta_{\mathrm{base}}$ & eligibility condition \\
$\rho_i$, $\kappa_i$ & reliability factor and source scale & $\mathcal A_{\tau}$, $\mathcal V$ & score-selection and admission maps \\
$\mathbf e_i$ & nonnegative evidence vector of source $i$ & $y^{(0)}$ & preliminary selection decision \\
$W$, $\mathbf a$ & prior strength and base rate & $\sigma^{(0)}$, $\varrho^{(0)}$ & preliminary selection status and typed response \\
$\preceq$, $\wedge$ & coordinatewise order and meet & \makecell[l]{$f_{\mathrm{cur}},f_{\mathrm{struct}},$\\$f_{\mathrm{risk}},f_{\mathrm{prov}}$} & failure indicators for command consistency, source validity, risk support, and component corroboration \\
\bottomrule
\end{tabularx}
\endgroup
\end{table}

% =============================================
% APPENDIX A.2: Proofs of the main-text propositions
% =============================================
\subsection{Proofs of the main-text propositions}

% ==================================================
% APPENDIX A.2.1: Auxiliary definitions and decision maps
% ==================================================
\subsubsection{Auxiliary definitions and decision maps}

For the benchmark evaluation, the reliability factor is
\begin{equation}
\rho_i
:=
[q_i]_0^1\bigl(1-[d_i]_0^1\bigr),
\qquad
[a]_0^1:=\min\{1,\max\{0,a\}\}.
\label{eq:pact_reliability}
\end{equation}
The reliability mapping in \Cref{eq:pact_reliability} is fixed. The shared fusion concentration and the source-specific scaling used only in the sensitivity analysis are selected on training data and are then transferred to the held-out instances.

\begin{definition}[Provenance-aware evidence conservation]
\label{def:provenance_evidence_conservation}
Consider any extension of the source set that preserves all existing evidence vectors and parent assignments. A fusion operator satisfies provenance-aware evidence conservation when two conditions hold for every such extension. Adding a source that joins the existing provenance structure cannot increase cumulative evidence coordinatewise, while adding an exact duplicate with matching parents preserves cumulative evidence. A source that forms a disconnected component may contribute through cross-component accumulation.
\end{definition}

For $s\in[0,1]$ and $\tau\in[0,1]\cup\{+\infty\}$, the exact preliminary selection map is
\begin{equation}
\mathcal A_{\tau}
\bigl(\widehat k,s,\zeta_{\mathrm{base}}\bigr)
:=
\begin{cases}
(\mathrm{selected},\widehat k,\mathrm{none}),
&
s\geq\tau \ \land\ \zeta_{\mathrm{base}}=1,
\\
(\mathrm{withheld},\widehat k,\mathrm{hold}),
&
\text{otherwise}.
\end{cases}
\label{eq:pact_preliminary_selection}
\end{equation}

Write the preliminary decision as $y^{(0)}=(\sigma^{(0)},\widehat k,\varrho^{(0)})$, where $\sigma^{(0)}$ is the selection status and $\varrho^{(0)}$ is the typed response. The indicators $f_{\mathrm{cur}}$, $f_{\mathrm{struct}}$, $f_{\mathrm{risk}}$, and $f_{\mathrm{prov}}$ correspond, respectively, to failures of command consistency, source validity, risk support, and component corroboration in \Cref{tab:typed_admission_policy}. Each equals one exactly when its associated condition fails. The typed admission map uses the first applicable case in the order shown.
\begin{equation}
\mathcal V(y^{(0)},x)
:=
\begin{cases}
y^{(0)},
&
\sigma^{(0)}\neq\mathrm{selected},
\\
(\mathrm{withheld},\widehat k,\mathrm{hold}),
&
f_{\mathrm{cur}}(x)=1,
\\
(\mathrm{withheld},\widehat k,\mathrm{fallback}),
&
f_{\mathrm{struct}}(x)=1,
\\
(\mathrm{withheld},\widehat k,\mathrm{hold}),
&
f_{\mathrm{risk}}(x)=1,
\\
(\mathrm{withheld},\widehat k,\mathrm{confirm}),
&
f_{\mathrm{prov}}(\widehat k;x)=1,
\\
(\mathrm{admitted},\widehat k,\mathrm{none}),
&
\text{otherwise}.
\end{cases}
\label{eq:pact_typed_verifier}
\end{equation}

If an exact within-component copy inherits the original evidence vector and parent assignment without altering base eligibility or any remaining admission condition, then $\mathbf E_{\Pi}$, $\boldsymbol\pi$, $s$, $\widehat k$, $y^{(0)}$, and the final decision $Y(x)$ all remain unchanged.

Let $P_{\Sigma}=\sum_i|P_i(x)|$ and $P_{\max}=\max_i|P_i(x)|$. Reading and hash-indexing the parent sets requires expected $O(P_{\Sigma})$ time and $O(P_{\Sigma})$ space. With expected $O(1)$ time per hash-set membership query, overlap discovery and aggregation require expected
\[
O\!\left(P_{\Sigma}+\sum_{i<j}\min\{|P_i(x)|,|P_j(x)|\}+JK\right)
\]
time and $O(J^2+JK+JP_{\max})$ space. Once overlap discovery is complete, storing only the discovered edges reduces connected-component traversal and evidence aggregation to $O(J+|\mathcal E_P(x)|+JK)$ time and space. The bounded-cardinality and sparse-graph reductions stated in \Cref{subsec:selection_properties} follow directly.

% ==================
% APPENDIX A.2.2: Proofs
% ==================
\subsubsection{Proofs}

\begin{proof}[Proof of \Cref{prop:opinion_only_indistinguishability}]
Both inputs present $\{\boldsymbol\alpha,\boldsymbol\alpha\}\in\mathcal M(\mathbb A)$ to $H$ and therefore yield the same output.
\end{proof}

\begin{proof}[Proof of \Cref{prop:pact_within_component_characterization}]
Fix $\mathbf v\in\mathcal E$. Starting from $\{\mathbf v\}$ and inserting all remaining occurrences yields $g(\mathcal E)\preceq\mathbf v$. Since $\mathbf v$ was arbitrary, $g(\mathcal E)\preceq\bigwedge_{\mathbf v\in\mathcal E}\mathbf v$. The meet itself satisfies both requirements and is therefore the unique pointwise greatest admissible rule.
\end{proof}

\begin{proof}[Proof of \Cref{prop:pact_common_budget}]
For each component $C$, admissibility gives $\mathbf g_C\preceq\mathbf b_C$, and $\mathbf b_C$ itself is admissible. Summing over the components proves greatestness. If an inserted source connects a family $\mathcal C_{\mathrm{join}}$ of existing components, their previous contribution is $\sum_{C\in\mathcal C_{\mathrm{join}}}\mathbf b_C$. The budget of the new merged component is the meet of the evidence vectors of the inserted source and of all members of these components. It is coordinatewise no greater than any $\mathbf b_C$ for $C\in\mathcal C_{\mathrm{join}}$, and therefore no greater than their nonnegative sum. An exact copy with matching parents remains in the original component. Every source sharing one of those parents was already connected to the original source, so the copy cannot connect previously distinct components, and its duplicate vector preserves the meet. A disconnected component contributes the nonnegative term $\mathbf b_D$.
\end{proof}

\begin{proof}[Proof of \Cref{cor:pact_partition_monotonicity}]
Each block of $\Pi_2$ is a union of blocks of $\Pi_1$. Its coordinatewise meet is no greater than the meet of any constituent fine block and therefore no greater than their nonnegative sum. Summing over the disjoint coarse blocks yields \eqref{eq:pact_partition_monotonicity}.
\end{proof}

% ===============================================
% APPENDIX A.3: Stability under a given partition
% ===============================================
\subsection{Stability under a given partition}

\begin{theorem}[Distribution and score stability under a given partition]
\label{thm:pact_partition_stability}
Fix the source index set, provenance partition $\Pi_P$, prior strength $W>0$, and base rate $\mathbf a$. For two evidence collections $\{\mathbf e_i\}$ and $\{\mathbf e_i'\}$, define
\begin{equation}
\Delta_{\Pi}
=
\sum_{C\in\Pi_P}
\sum_{k=1}^{K}
\max_{i\in C}
|e_{ik}'-e_{ik}|.
\label{eq:pact_stability_delta}
\end{equation}
Let $Z=W+\lVert\mathbf E_{\Pi}\rVert_1$ and $Z'=W+\lVert\mathbf E_{\Pi}'\rVert_1$. Then
\begin{equation}
\lVert\mathbf E_{\Pi}'-\mathbf E_{\Pi}\rVert_1
\leq
\Delta_{\Pi},
\qquad
\lVert\boldsymbol\pi'-\boldsymbol\pi\rVert_1
\leq
\frac{2\Delta_{\Pi}}{\max\{Z,Z'\}}
\leq
\frac{2\Delta_{\Pi}}{W},
\label{eq:pact_distribution_stability}
\end{equation}
and
\begin{equation}
|s'-s|
\leq
\frac{W\Delta_{\Pi}}{ZZ'}
\leq
\frac{\Delta_{\Pi}}{W}.
\label{eq:pact_score_stability}
\end{equation}
\end{theorem}

\begin{proof}
Within each component, the scalar minimum in each coordinate is nonexpansive under perturbations of its arguments, and summing these coordinatewise bounds gives the first inequality. Let $\mathbf x=W\mathbf a+\mathbf E_{\Pi}$ and $\mathbf x'=W\mathbf a+\mathbf E_{\Pi}'$. Since $\mathbf a$ is a base-rate vector and all evidence vectors are nonnegative, the coordinate sums of $\mathbf x$ and $\mathbf x'$ are $Z$ and $Z'$, respectively. The difference between these numerators equals the difference in cumulative evidence, so its norm is bounded by $\Delta_{\Pi}$. The reverse triangle inequality gives $|Z-Z'|\leq\sum_{k\in\mathcal K}|x_k-x_k'|$. If $Z\geq Z'$, then
\begin{equation}
\left\lVert
\frac{\mathbf x}{Z}
-
\frac{\mathbf x'}{Z'}
\right\rVert_1
\leq
\frac{\lVert\mathbf x-\mathbf x'\rVert_1}{Z}
+
\frac{|Z-Z'|}{Z}
\leq
\frac{2\lVert\mathbf x-\mathbf x'\rVert_1}{Z}.
\end{equation}
The case $Z'\geq Z$ follows symmetrically. Writing $E_{\mathrm{tot}}=\lVert\mathbf E_{\Pi}\rVert_1$ gives $s=E_{\mathrm{tot}}/(W+E_{\mathrm{tot}})$ and $|s'-s|=W|E'_{\mathrm{tot}}-E_{\mathrm{tot}}|/[(W+E_{\mathrm{tot}})(W+E'_{\mathrm{tot}})]$. The first evidence bound gives $|E'_{\mathrm{tot}}-E_{\mathrm{tot}}|\leq\Delta_{\Pi}$, which yields the stated score bound.
\end{proof}

With the source set and the provenance partition fixed, the posterior and the selection score satisfy Lipschitz bounds for perturbations measured by $\Delta_{\Pi}$, with global constants $2/W$ and $1/W$, respectively. These global constants decrease as the prior strength $W$ increases, although the sharper score bound above need not be monotone in $W$.

For $\mathbf x\in\mathbb R^K$, let $(\mathbf x)_+$ denote its coordinatewise positive part, with $[(\mathbf x)_+]_k=\max\{x_k,0\}$.

\begin{corollary}[Within-component near-copy insertion]
\label{cor:pact_near_copy}
Insert $\mathbf e_{+}$ into an existing component $C$ without merging previously distinct components, and define $\varepsilon_{+}=\lVert(\mathbf b_C-\mathbf e_{+})_{+}\rVert_1$. Then
\begin{equation}
\mathbf E_{\Pi}^{+}
=
\mathbf E_{\Pi}
-
(\mathbf b_C-\mathbf e_{+})_{+},
\qquad
\lVert\boldsymbol\pi^{+}-\boldsymbol\pi\rVert_1
\leq
\frac{2\varepsilon_{+}}{W},
\qquad
s^{+}\leq s.
\label{eq:pact_near_copy_bound}
\end{equation}
If $\lVert\mathbf e_{+}-\mathbf e_j\rVert_1\leq\epsilon$ for some $j\in C$, then $\varepsilon_{+}\leq\epsilon$.
\end{corollary}

The budget identity follows coordinatewise from $\mathbf b_C\wedge\mathbf e_{+}=\mathbf b_C-(\mathbf b_C-\mathbf e_{+})_{+}$. The normalization bound in the proof of \Cref{thm:pact_partition_stability} depends only on the posterior numerators and on their coordinate sums, so it applies to this change of budget. Monotonicity of the selection score in the total evidence budget gives $s^{+}\leq s$, and $\mathbf b_C\preceq\mathbf e_j$ yields the final inequality.

\begin{corollary}[Margin-conditional prediction and score-selection stability]
\label{cor:pact_selection_stability}
Suppose $\widehat k$ is the unique maximizer of $\boldsymbol\pi$ with margin $\mu_{\pi}=\pi_{\widehat k}-\max_{j\neq\widehat k}\pi_j>0$. Sufficient conditions for preserving both the predicted contract and the score-selection status are
\begin{equation}
\frac{2\Delta_{\Pi}}{\max\{Z,Z'\}}<\mu_{\pi},
\qquad
\frac{W\Delta_{\Pi}}{ZZ'}<|s-\tau|.
\end{equation}
The first condition preserves the predicted contract. For fixed $\tau$ and $\zeta_{\mathrm{base}}$, the second preserves the select--withhold status. Taken together, they preserve the selection-stage decision $y^{(0)}$.
\end{corollary}

The posterior condition keeps every top-versus-competitor gap positive under the stated perturbation bound, while the score condition prevents $s$ from crossing $\tau$.

% =============================================
% APPENDIX B: Additional Experimental Details
% =============================================
\section{Additional Experimental Details}
\label{sec:additional_experimental_details}

\subsection{Benchmark and comparator specifications}
\begin{table}[pos=htbp]
\centering
\caption{Outcomes of alternative component-corroboration rules.}
\label{tab:component_corroboration_rules}
\begin{threeparttable}
\IFTableSetup
\begingroup
\footnotesize
\renewcommand{\arraystretch}{1.00}
\setlength{\tabcolsep}{2.5pt}
\begin{tabularx}{\linewidth}{@{}>{\centering\arraybackslash}p{1.25cm}>{\centering\arraybackslash}p{1.55cm}>{\raggedright\arraybackslash}X>{\centering\arraybackslash}p{1.35cm}>{\centering\arraybackslash}p{1.55cm}>{\centering\arraybackslash}p{1.35cm}@{}}
\toprule
\makecell[c]{Supporting\\sources} & \makecell[c]{Complete\\components} & Supporting roles & Unanimity & \makecell[c]{Component\\count} & Cross-role \\
\midrule
1 & 1 & Any & -- & -- & -- \\
\multirow{3}{*}{2} & 1 & Language + geometry/risk & -- & \textsc{confirm} & \textsc{confirm} \\
& 2 & Language + geometry/risk & -- & \textsc{confirm} & \textsc{admit} \\
& 2 & No language + geometry/risk pair & -- & \textsc{confirm} & \textsc{confirm} \\
\multirow{3}{*}{3} & 1 & Language + geometry/risk & \textsc{confirm} & \textsc{confirm} & \textsc{confirm} \\
& 2 & Language + geometry/risk & \textsc{confirm} & \textsc{confirm} & \textsc{admit} \\
& 3 & Language + geometry/risk & \textsc{admit} & \textsc{admit} & \textsc{admit} \\
\bottomrule
\end{tabularx}
\endgroup
\begin{tablenotes}[flushleft]\footnotesize
\item[] \textit{Note.} Unanimity applies to high-confidence support from all three sources. Component count requires at least two supporting sources and admits only at the complete-component threshold. Cross-role admits when language and geometry or path risk support the candidate in two complete components. A dash denotes a rule that is not invoked.
\end{tablenotes}
\end{threeparttable}
\end{table}

Removing individual checks changes admission in different ways. With the nested-Dirichlet candidates and scores held constant, removing command consistency adds 3,256 wrong admissions, whereas removing source validity, risk support, or component corroboration adds 2,400, 0, or 2,400, respectively. The risk-support rejections observed elsewhere arise under provenance-discounted pooling. For the fixed nested-Dirichlet candidates, the admission rule admits all 2,400 constructed adversarial-consensus cases with three supporting components and maps them to confirmation with two. The same rule maps these cases to fallback where geometry is incomplete, and maps all 4,062 otherwise admitted candidates to \textsc{hold} where language is marked stale.

Admission outcomes depend on the manner in which supporting components are counted. Under the unanimity rule, PACT and nested Dirichlet obtain ncsAURC values of 0.086 and 0.148. The component-count variant closely matches them at 0.087 and 0.149 while retaining 18 fewer correct candidates per method. The cross-role rule raises ncsAURC to 0.389 and 0.412 and yields $R_{\mathrm{all}}=7.69\%$ at near-matched coverage. For PACT, these values coincide with those of the provenance-partition configuration without the component-corroboration check in \Cref{tab:selective_performance}, although the two rules map two-source cases differently. Under the cross-role rule, the 2,400 additional wrong admissions obtained with PACT have three agreeing source outputs but only two complete provenance components.

Under the illustrative decision cost in \Cref{eq:decision_cost}, an always-withhold policy has $C_{10}=1$. The unanimity rule attains 0.870 (95\% CI $[0.864,0.876]$), as against 6.311 ($[6.303,6.318]$) for always continuing with the candidate produced by PACT. Reducing the component requirement to $\nu=2$ raises the cost to 1.639 ($[1.625,1.653]$), which exceeds the always-withhold cost. These comparisons use the illustrative 10:1 relative cost and do not estimate application-specific utility. A shared non-release cost also does not distinguish the costs of hold, confirmation, and fallback.

Among instances with no unavailable source, corroboration strengthens the partition contrast from $-0.0025$ (95\% CI $[-0.0029,-0.0021]$) to $-0.0276$ ($[-0.0298,-0.0254]$). Over $[0.10,0.35]$, the corresponding contrasts are $-0.00017$ ($[-0.00038,0.00005]$) and $-0.0216$ ($[-0.0234,-0.0198]$).

The benchmark maps language, handover geometry, and path risk to opinions over the five action contracts. Close-to-hand geometry and increasing path risk shift support away from unconstrained transfer and toward slower, withheld, or corrective responses. Occlusion lowers source quality and increases conflict. Near-copy perturbations use $\epsilon\in\{0.001,0.005,0.01,0.02\}$ and move $\epsilon$ probability mass from the leading geometry contract to the second-highest geometry contract while provenance is preserved. The multiplicity study repeats the geometry output within the same component. Missing-source sensitivity compares retention of an unavailable source as zero evidence with omission of that source before component construction. In this benchmark, omission, source removal, and null-source isolation are equivalent, since the unavailable source bridges no observed components.

The admission summaries in \Cref{tab:typed_admission_policy} are defined over the observed sources. For source $i$, $p_i^{\mathrm{peak}}=\max_{k\in\mathcal K}p_{ik}$ and $\Delta_i=p_{i,h_i(x)}-\max_{j\neq h_i(x)}p_{ij}$. Where $\mathcal I_{k^*}(x)$ is nonempty, $p_{\min}^{\mathrm{peak,obs}}$, $q_{\min}^{\mathrm{obs}}$, and $d_{\max}^{\mathrm{obs}}$ denote the minimum peak probability, the minimum quality, and the maximum conflict among the sources supporting the selected candidate $k^*$.

The primary benchmark specification uses $K=W=5$ and equal source scales. The nested-Dirichlet concentration is selected within each outer training fold and is then shared with PACT. Hierarchy-matched cautious fusion participates in the outer-fold comparison but not in concentration selection. Source-specific rescaling is evaluated separately as a sensitivity analysis.

Quality-weighted fusion averages the available distributions with weights proportional to $q_i$. Product fusion discounts each distribution toward the uniform base rate, $r_{ik}=\rho_i p_{ik}+(1-\rho_i)/K$, multiplies the discounted probabilities for each class across sources, and normalizes across classes. Nested Dirichlet composition forms source evidence using the fold-selected concentration $\kappa$, and PACT uses the same concentration under equal source scaling:
\begin{equation*}
e_{ik}=\rho_i\kappa p_{ik},
\qquad
b_{ik}=\frac{e_{ik}}{K+\sum_\ell e_{i\ell}},
\qquad
u_i=\frac{K}{K+\sum_\ell e_{i\ell}}.
\end{equation*}

For nested Dirichlet, the available opinions are composed sequentially in the fixed order $L,G,R$ using
\begin{equation*}
c=\sum_{k\neq\ell}b_kb'_\ell,
\qquad
\widetilde b_k=\frac{b_kb'_k+b_ku'+b'_ku}{1-c},
\qquad
\widetilde u=\frac{uu'}{1-c}.
\end{equation*}
This is followed by $\pi_k=\widetilde b_k+\widetilde u/K$ under the TMC parameterization. The corresponding selection score is $1-\widetilde u$~\citep{Han2021TMC}. Global cautious fusion applies the cautious rule of Den{\oe}ux to all available masses without a provenance partition~\citep{Denoeux2008Cautious}.

Hierarchy-matched cautious fusion takes the cautious minimum within each provenance component before conjunctive combination across components. Unavailable sources contribute vacuous mass, and the pignistic transformation supplies the decision probabilities~\citep{SmetsKennes1994TBM}. For pooling, source $i$ receives log-weight $2q_i-2d_i-4m_i+\log\delta_i$, where $m_i$ marks an unavailable source and $\delta_i=(1+\eta_i)^{-1}$ discounts the source weight according to the total edge weight $\eta_i$ to connected sources with the same predicted class. The normalized weights define the weighted log pool, and the value-only counterpart sets $\delta_i=1$. The remaining comparators omit unavailable sources, while nonmissing invalid distributions are mapped to the uniform distribution. If every source is unavailable, quality-weighted fusion and product fusion return the uniform distribution. Global cautious and hierarchy-matched cautious fusion reduce the empty or vacuous collection to a uniform pignistic distribution. Nested Dirichlet returns the uniform distribution with native score zero, and the instance is ineligible.

Score comparisons use posterior peak, top-two margin, inverse normalized entropy, and pooled correctness. For pooled correctness, one standardized $\ell_2$-regularized logistic model (penalty $10^{-4}$) is fitted in each outer fold to the pooled outer-training predictions from product fusion, nested Dirichlet, global cautious fusion, and PACT. Its inputs are the five posterior probabilities, normalized entropy, posterior peak, top-two margin, and observed-source count, and method identity is excluded. The fitted model and the standardization of its training fold score the held-out predictions from those four methods and, without refitting of the score model, from provenance-discounted pooling and hierarchy-matched cautious fusion. Score thresholds are fitted separately on the outer-training instances of each method. The fixed-representative ablation retains the first source in each component under the common source order and recovers the budget only where $\mathbf e_{i^\star}=\mathbf b_C$.

Relative to the unduplicated output, product fusion and nested Dirichlet change the predicted contract in 33.4\% and 34.0\% of the exact-copy comparisons, respectively. Provenance-discounted pooling preserves the predicted contract yet reaches a posterior $\ell_1$ drift of 0.227. For the undegraded near-copy sweep, a descriptive zero-intercept linear fit of posterior $\ell_1$ drift against $\epsilon$ has slope 0.216.

For the fixed-representative ablation, the selection score changes on 82.3\% of all instances and on 89.9\% of the instances for which the representative and the meet predict the same contract.

At a provenance-misspecification rate of 50\%, the wrong-admission rate under false refinement remains 43.8\% below the singleton reference and reaches that reference only at 100\%, the largest rate tested. Merging and recoverable deletion remain below the singleton reference throughout the tested grid. Incorrect distinct-parent assignments and false refinement produce identical outputs under this construction.

For the stratification by availability and validity, product fusion has a pre-admission ncsAURC of 0.747 on the 19,200 three-valid-source instances and of 0.952 on the 12,000 instances with an unavailable or invalid source. Restricting eligibility to all-valid instances while the full denominator is retained gives 0.616, 0.699, and 0.635 for PACT, nested Dirichlet, and product fusion, respectively.

Within the 19,200 instances with three valid sources, the native-score post-admission ncsAURC is 0.0012 for PACT and 0.0199 for nested Dirichlet. Under posterior-peak ranking, the corresponding values are 0.0956 and 0.0979, while product fusion reaches 0.0909. Once the constructed adversarial-consensus condition is removed, PACT and nested Dirichlet both report 0.0532 under posterior-peak ranking, while product fusion reaches 0.0517 on the remaining 16,800 three-valid-source instances.

Across the 856 single-merge relations, coarsening increases selection-score ncsAURC for 73.2\% of these merges with TMC and for 73.6\% with RCML, decreases it for 17.1\% and 18.7\%, respectively, and leaves it unchanged for the remaining merges.

\begin{table}[pos=htbp]
\centering
\caption{Pooled correctness-score ncsAURC on the three-valid-source subset (19,200 instances). Lower values indicate lower risk over the common coverage support.}
\label{tab:pooled_score_valid_sources}
\begin{tabular}{lcc}
\toprule
Fusion rule & Before admission & After admission \\
\midrule
PACT & 0.803 & 0.367 \\
Nested Dirichlet & 0.432 & 0.048 \\
Product fusion & 0.405 & 0.232 \\
Global cautious fusion & 0.384 & 0.109 \\
\bottomrule
\end{tabular}
\end{table}

\subsection{Threshold and learned-comparator details}

Fusion-concentration candidates are ranked on outer-training data using the illustrative relative decision cost
\begin{equation}
C_{\lambda}
=
\frac{\lambda N_{\mathrm{wrong}}+N_{\mathrm{nonrelease}}}{N},
\qquad
\lambda=10,
\label{eq:decision_cost}
\end{equation}
where $N_{\mathrm{wrong}}$ and $N_{\mathrm{nonrelease}}$ count wrong admissions and non-release decisions among $N$ evaluations. For each concentration candidate, the cost is averaged across the target coverages $0.10$, $0.13$, and $0.15$. The fusion-concentration candidates are $4,8,12,16,24$. Ties are resolved by fewer wrong admissions and then by greater correct retention. Remaining ties follow a fixed order of quality-weighted fusion, product fusion, global cautious fusion, lineage-unaware pooling, provenance-discounted pooling, and nested Dirichlet with concentration $4,8,12,16,24$, in that order.

For method $b$, outer fold $f$, and stage $a\in\{\mathrm{pre},\mathrm{post}\}$, let $g_b$ denote the selection score, $\mathcal D_f^{\mathrm{tr}}$ the full outer-training set, and $\mathcal D_{f,b,\mathrm{elig},a}^{\mathrm{tr}}$ its base-eligible subset. Where $a=\mathrm{post}$, this subset is further restricted to the instances that are admission-eligible under method $b$. At the target coverage $\gamma_{\mathrm{tar}}$, the cutoff is
\begin{equation}
r_f=\operatorname{round}_{\mathrm{even}}\!\left(\gamma_{\mathrm{tar}}|\mathcal D_f^{\mathrm{tr}}|\right),\quad
j_{f,b,a}=\min\{|\mathcal D_{f,b,\mathrm{elig},a}^{\mathrm{tr}}|,r_f\},\quad
\tau_{f,b,a}=
\begin{cases}
g_{b,(j_{f,b,a})},&j_{f,b,a}\geq1,\\
+\infty,&j_{f,b,a}=0.
\end{cases}
\label{eq:selective_threshold}
\end{equation}

Here $\operatorname{round}_{\mathrm{even}}$ rounds to the nearest integer, with half-integer ties resolved to the even integer, and $g_{b,(j_{f,b,a})}$ is the $j_{f,b,a}$th largest score among the instances eligible for method $b$ at stage $a$. The target count is defined relative to the full outer-training set and is capped by the number of eligible instances. Instances tied at the cutoff are retained. Scores lie in $[0,1]$, so the extended cutoff $+\infty$ for an empty eligible set selects no instance.

Risk--coverage curves use 36 equally spaced coverage points, linear interpolation, and trapezoidal integration. For PACT, the base-eligible subset contains 28,800 instances, of which 17,769 are misclassified, which gives the random-ordering reference $A_{\mathrm{PACT}}^{\mathrm{rand}}=17{,}769/28{,}800\approx0.616979$. Confidence intervals are contrast-specific and are unadjusted for multiplicity.

For the quality of the posterior probabilities, the unnormalized multiclass Brier score is $\sum_{k=1}^{K}(\pi_k-\mathbb 1\{k=k_{\mathrm{pref}}\})^2\in[0,2]$, where $k_{\mathrm{pref}}$ denotes the reference contract. ECE$_{10}$ uses the posterior peak as confidence and ten equal-width bins weighted by empirical frequency. Empty bins contribute zero.

Let $A_b$ denote the observed ncsAURC of the combination $b$ of method and selection score. The ncsAURC gap between product fusion and PACT comprises a difference under random ordering and a difference in excess risk above the respective random-ordering references.
\begin{equation*}
\begin{aligned}
A_{\mathrm{prod}}-A_{\mathrm{PACT}}
&=
\bigl(A_{\mathrm{prod}}^{\mathrm{rand}}-A_{\mathrm{PACT}}^{\mathrm{rand}}\bigr)
+
\bigl[(A_{\mathrm{prod}}-A_{\mathrm{prod}}^{\mathrm{rand}})
-(A_{\mathrm{PACT}}-A_{\mathrm{PACT}}^{\mathrm{rand}})\bigr]
\\
&\approx0.0725+0.0615\approx0.1340.
\end{aligned}
\end{equation*}
Of the gap of 0.1340, a portion of 0.0725 (54.1\%) is already present under random ordering and reflects differences in the candidates produced by the fusion rules. Relative to their own random-ordering references, PACT and product fusion have excess ncsAURC values of 0.0125 (95\% CI $[0.0053,0.0188]$) and 0.0740. The product-minus-PACT difference in excess ncsAURC accounts for the remaining 0.0615. Neither native score improves correctness ordering over its own random-ordering reference in this benchmark.

In the joint scene-and-condition holdout, the separately trained Set Transformer has an ncsAURC of 0.558 (95\% CI $[0.536,0.580]$) and performs worse than each analytic rule in that comparison.

A Set Transformer receives a 19-dimensional representation of each source. It includes five class probabilities, quality and conflict scores, and indicators of missingness, staleness, invalid source format, and missing conflict information. An invalid-format indicator marks an explicitly invalid source specification, while the missing-conflict indicator marks an absent conflict value in an otherwise present source. Three further indicators describe whether the source is current, whether stale evidence carries the previous command, and whether the nonempty current-command hash differs from the nonempty hash stored with the source evidence. Three source-type indicators and two lineage-group indicators complete the representation.

The model uses two four-head self-attention blocks of width 96, feed-forward width 192 with rectified linear unit activations, and a learned four-head pooling query. Separate heads predict the action contract and the selection score. Training minimizes selection-weighted cross-entropy with denominator floor $10^{-6}$. The loss also includes a coefficient-32 penalty where mean selection falls below 0.40, together with a coefficient-0.5 full-sample cross-entropy term. AdamW uses learning rate 0.003, weight decay 0.0001, and batches of 1,024 for at most 60 epochs. Early stopping monitors the mean validation selective loss and stops after eight epochs without an improvement of $10^{-5}$. Results are averaged over three initializations for each held-out condition and outer fold. Training, model selection, and cutoff fitting exclude the test scenes and the held-out source condition.

Each of the 13 source conditions contributes 2,400 of the 31,200 instances. The conditions cover source loss or invalidity, degraded quality, stale context, unsafe urgency, adversarial agreement, missing conflict information, partial disagreement, and high-confidence contradiction. These source conditions vary the local evidence state, whereas the matched provenance interventions separately vary the manner in which support is counted.

\subsection{Learned-prediction evaluation details}

HandWritten supports budget, posterior, accuracy, and ordering analyses for TMC and RCML. The PIE--RCML pair supports budget analysis only. The Scene15--TMC pair adds a second dataset for the full provenance, accuracy, and selective-ordering analysis.

Across the HandWritten and PIE dataset--model pairs, the largest false-splitting expansion occurs on the PIE--RCML pair ($3.21\times$ the original-view budget), while the strongest all-view contraction occurs on the HandWritten--TMC pair (0.16\% of the original-view budget).

\begin{table}[pos=htbp]
\centering
\caption{Effects of camera grouping with four prompt variants.}
\label{tab:native_view_fm_transfer}
\begin{threeparttable}
\IFTableSetup
\begingroup
\footnotesize
\renewcommand{\arraystretch}{0.98}
\setlength{\tabcolsep}{2.5pt}
\begin{tabular*}{\linewidth}{@{}l@{\extracolsep{\fill}}
S[table-format=1.4]
S[table-format=1.4]
S[table-format=2.2]
S[table-format=1.3]
S[table-format=1.3]@{}}
\toprule
& \multicolumn{3}{c}{Predictive and selective performance}
& \multicolumn{2}{c}{Structural response} \\
\cmidrule(lr){2-4}\cmidrule(lr){5-6}
Condition
& {\makecell[c]{Selection\\ncsAURC$\downarrow$}}
& {\makecell[c]{Posterior\\ncsAURC$\downarrow$}}
& {Acc. (\%)$\uparrow$}
& {\makecell[c]{Budget\\ratio}}
& {\makecell[c]{Posterior\\$\ell_1$ drift}} \\
\midrule
\multicolumn{6}{@{}l}{\textbf{Qwen3-VL-32B}} \\
One output per camera (ref.) & 0.2554 & 0.2064 & 74.46 & 1.000 & 0.000 \\
Acquisition grouped          & 0.2257 & 0.1702 & 78.35 & 0.783 & 0.145 \\
Per-output counting          & 0.2481 & 0.1881 & 75.19 & 4.000 & 0.159 \\
Identical-vector merge       & 0.2661 & 0.1845 & 75.59 & 3.839 & 0.158 \\
All-view merge               & 0.2327 & 0.2380 & 54.93 & 0.038 & 0.298 \\
\midrule
\multicolumn{6}{@{}l}{\textbf{Qwen3-VL-8B}} \\
One output per camera (ref.) & 0.4675 & 0.2524 & 53.25 & 1.000 & 0.000 \\
Acquisition grouped          & 0.2562 & 0.2560 & 59.22 & 0.476 & 0.213 \\
Per-output counting          & 0.3764 & 0.2404 & 62.36 & 4.000 & 0.275 \\
Identical-vector merge       & 0.4315 & 0.2426 & 61.30 & 3.677 & 0.276 \\
All-view merge               & 0.2570 & 0.2551 & 57.71 & 0.028 & 0.492 \\
\bottomrule
\end{tabular*}
\endgroup
\begin{tablenotes}[flushleft]\footnotesize
\item[] \textit{Note.} Acc. denotes accuracy. Metrics are computed within task and averaged equally across the six tasks. Both ncsAURC columns use common support $[0.10,0.90]$. Budget ratio and posterior $\ell_1$ drift are relative to one output per camera. Budget ratios remain checkpoint-specific because the meet depends on checkpoint predictions.
\end{tablenotes}
\end{threeparttable}
\end{table}

Checkpoint separation and all-checkpoint merging produce posterior $\ell_1$ drifts of 0.006 and 0.095, respectively, relative to model-family grouping.

The checkpoint comparison uses pooled event-window aggregation over the common support $[0.10,0.90]$. Under model-family grouping, PACT fusion reaches 86.6\% accuracy, close to the 86.5\% obtained by nested Dirichlet and by hierarchy-matched cautious fusion. Product fusion has the lowest NLL (0.441 against 0.596 for PACT fusion). Across product fusion, nested Dirichlet, hierarchy-matched cautious fusion, PACT fusion with model-family grouping, and PACT fusion with checkpoint separation, ncsAURC spans 0.1642--0.1654. The value for nested Dirichlet is lower than that for model-family-grouped PACT by 0.0007 (95\% CI $[0.0003,0.0011]$). The operators are close in accuracy and in ncsAURC, yet they differ in probability quality and in calibration.

Under pooled event-window weighting, the model-family pairing ranks seventh among 105 pairings, with nine tied or higher ($9/105\approx0.086$). Equal task weighting moves it to 22nd, with 24 tied or higher ($24/105\approx0.229$). A post-hoc sensitivity analysis excludes the two LLaVA checkpoints, which predict \textsc{ready} for nearly every event. The model-family pairing then ranks third among 15 pairings, which leaves the association unconfirmed (one-sided tail fraction $3/15=0.20$).

For HABIT, the log odds under both label mappings are oriented as \textsc{ready} minus \textsc{not ready} and averaged. Applying the sigmoid to the averaged log odds gives $p_{\mathrm{ready}}$, which yields evidence $2(1-p_{\mathrm{ready}},p_{\mathrm{ready}})$. In the admission study, the evaluation comprises 120 episode-target and 600 counterfactual-target queries per checkpoint across the six tasks. The separate checkpoint panel pairs the 8B/32B Qwen3-VL, 2B/8B InternVL3, 0.5B/7B LLaVA-OneVision, and 0.5B/2.2B SmolVLM2 variants. The 105 partitions of the eight checkpoints into four pairs are ranked by mean within-pair binary-error $\phi$ minus mean between-pair binary-error $\phi$. Bootstrap resampling draws episodes within tasks while all windows and checkpoints are retained.

The common-score comparison uses acquisition-level records from the two Qwen3-VL checkpoints. It holds the numerical $m=4$ acquisition-grouped mass score fixed while the candidate produced by each counting rule is evaluated. Risk is evaluated at 36 target coverages from 0.10 to 0.90, averaged first within tasks and then equally across tasks, and integrated over that support. At a coverage boundary inside a tied-score group, the group contributes fractionally using its mean error rate. Paired 95\% percentile intervals use 2,000 episode-stratified bootstrap resamples, with both methods retained for every sampled episode.

Event proximity is estimated by regularized logistic regression over frozen ResNet-50 visual features~\citep{He2016ResNet} and term frequency--inverse document frequency text features. Regularization is selected by episode-grouped cross-validation. The primary model uses all development episodes, and each task-held-out model excludes the evaluated task. Absolute interfaces use a threshold of 0.5, development-selected thresholds, or Platt scaling. Positive change in the event-proximity score is indicated by $\mathbf 1\{s_{\mathrm{proximal}}>s_{\mathrm{early}}\}$, with ties assigned zero. Estimates are averaged within episodes and equally across tasks. Admission and temporal-control intervals use task-stratified episode bootstraps.

\begin{table}[pos=htbp]
\centering
\caption{Reference consistency and admission retention across checkpoints.}
\label{tab:checkpoint_admission_retention}
\begin{threeparttable}
\IFTableSetup
\begingroup
\footnotesize
\renewcommand{\arraystretch}{0.98}
\setlength{\tabcolsep}{2.0pt}
\begin{tabular*}{\linewidth}{@{}l@{\extracolsep{\fill}}*{6}{S[table-format=2.1]}l@{}}
\toprule
& \multicolumn{3}{c}{Candidate evidence only}
& \multicolumn{3}{c}{Identity + learned event proximity}
& \multicolumn{1}{c}{$\Delta$ inconsistent rate} \\
\cmidrule(lr){2-4}\cmidrule(lr){5-7}
Checkpoint
& {Consistent}
& {Inconsistent}
& {Ready recall (\%)}
& {Consistent}
& {Inconsistent}
& {Coverage (\%)}
& {pp (95\% CI)} \\
\midrule
Qwen3-VL-8B  & 53 & 91 & 88.3 & 51 & 1 & 7.2 & $-12.50\;[-13.75,-11.25]$ \\
Qwen3-VL-32B & 57 & 86 & 95.0 & 54 & 1 & 7.6 & $-11.81\;[-13.19,-10.42]$ \\
InternVL3-8B  & 27 & 71 & 45.0 & 24 & 1 & 3.5 & $-9.72\;[-11.67,-7.78]$ \\
\bottomrule
\end{tabular*}
\endgroup
\begin{tablenotes}[flushleft]\footnotesize
\item[] \textit{Note.} Reference-consistent admissions and ready recall use the 60 reference-\textsc{ready} episode targets. Reference-inconsistent admissions are counted over 720 offline queries per checkpoint, and coverage is relative to the same 720 queries. Here pp denotes percentage points and CI denotes confidence interval. Signed differences compare the joint rule with candidate evidence alone and use task-stratified episode resampling. Reference consistency denotes agreement with the dataset action, not physical task success.
\end{tablenotes}
\end{threeparttable}
\end{table}

\Cref{tab:checkpoint_admission_retention} compares Qwen3-VL-8B, Qwen3-VL-32B, and InternVL3-8B. Target identity and learned event proximity jointly leave one reference-inconsistent admission per checkpoint, with coverage below 8\%. For Qwen3-VL-8B, they exclude 90 of 91 inconsistent admissions and retain 51 of 53 consistent admissions.

For joint fusion and admission, each checkpoint supplies 2,400 outputs from 60 episodes, two temporal windows, five cameras, and four prompt variants. Prompt multiplicity, camera removal, and exact-copy multiplicity vary repeated computation and acquired viewpoints separately. Task-stratified bootstrap intervals are computed by resampling intact episodes.

With per-output counting under the binary mapping, adding seven copies of mass two to each of the 20 prompt outputs can raise the event-window budget by $20\times7\times2=280$ units. Release still depends on the temporal comparison, not an evidence-mass cutoff.

For Hotel Towel, all ten Qwen3-VL-32B candidates satisfy target identity and have an early-window reference, yet none passes the PACT temporal comparison. Per-output counting admits one, while one-output-per-camera counting admits none. For Qwen3-VL-8B, candidate selection fails first in one episode. None of the remaining nine passes either the PACT or the per-output temporal comparison, whereas four pass where one output per camera determines the temporal posterior.

\Cref{tab:joint_policy_factorial} reports all combinations of the identity and temporal requirements. In the supplementary all-output-merging arm at four prompts per camera, the same complete typed admission rule retains 46 reference-consistent candidates from Qwen3-VL-32B and 42 from Qwen3-VL-8B, with no observed reference-inconsistent admission among 720 queries per checkpoint.

\begin{table}[pos=htbp]
\centering
\caption{Offline HRC admission across counting rules and admission requirements.}
\label{tab:joint_policy_factorial}
\begin{threeparttable}
\IFTableSetup
\begingroup
\footnotesize
\renewcommand{\arraystretch}{0.98}
\setlength{\tabcolsep}{2.4pt}
\begin{tabular*}{\linewidth}{@{}l@{\extracolsep{\fill}}*{4}{S[table-format=2.0]}cc@{}}
\toprule
& \multicolumn{2}{c}{Qwen3-VL-32B}
& \multicolumn{2}{c}{Qwen3-VL-8B}
& \multicolumn{2}{c}{Requirements} \\
\cmidrule(lr){2-3}\cmidrule(lr){4-5}\cmidrule(lr){6-7}
Counting rule
& {Consistent} & {Inconsistent}
& {Consistent} & {Inconsistent}
& Identity & Temporal \\
\midrule
All three rules & 57 & 86 & 53 & 91 & \xmark & \xmark \\
All three rules & 57 & 22 & 53 & 25 & \cmark & \xmark \\
\addlinespace[1pt]
\textbf{Camera-grouped PACT} & 47 & 38 & 43 & 51 & \multirow{3}{*}{\xmark} & \multirow{3}{*}{\cmark} \\
Per-output counting & 48 & 38 & 42 & 51 & & \\
One output per camera & 45 & 38 & 44 & 52 & & \\
\addlinespace[1pt]
\textbf{Camera-grouped PACT} & 47 & 0 & 43 & 0 & \multirow{3}{*}{\cmark} & \multirow{3}{*}{\cmark} \\
Per-output counting & 48 & 0 & 42 & 0 & & \\
One output per camera & 45 & 0 & 44 & 0 & & \\
\bottomrule
\end{tabular*}
\endgroup
\begin{tablenotes}[flushleft]\footnotesize
\item[] \textit{Note.} \cmark\ and \xmark\ denote enabled and disabled requirements, respectively. Temporal is the strict positive-change criterion. The row for all three rules gives identical counts for camera-grouped PACT, per-output counting, and one output per camera. Reference-consistent admissions use the 60 reference-\textsc{ready} episode targets. Reference-inconsistent admissions use all 720 queries per checkpoint. Candidate selection and target identity are fixed across counting arms. Consistency denotes agreement with the dataset reference action. Zero denotes no observed reference-inconsistent admission in this finite evaluation.
\end{tablenotes}
\end{threeparttable}
\end{table}

\begin{table}[pos=htbp]
\centering
\caption{Leave-one-task-out temporal-rule comparison under fixed candidates and target identity.}
\label{tab:habit_loto_temporal_transfer}
\begin{threeparttable}
\IFTableSetup
\begingroup
\footnotesize
\renewcommand{\arraystretch}{0.98}
\setlength{\tabcolsep}{2.6pt}
\begin{tabular*}{0.72\linewidth}{@{}l@{\extracolsep{\fill}}*{2}{S[table-format=2.0]}S[table-format=2.1]@{}}
\toprule
\multirow{2}{*}{Temporal rule}
& \multicolumn{2}{c}{Admissions}
& \multicolumn{1}{c}{\multirow{2}{*}{\makecell[c]{Reference-\textsc{ready}\\recall (\%)}}} \\
\cmidrule(lr){2-3}
& {Consistent} & {Inconsistent} & \\
\midrule
Absolute score (0.5) & 37 & 10 & 61.7 \\
Dev balanced accuracy & 34 & 10 & 56.7 \\
Dev FPR $\leq0.05$ & 38 & 10 & 63.3 \\
Dev Platt calibration & 37 & 10 & 61.7 \\
\textbf{Paired positive change} & \bfseries 53 & \bfseries 0 & \bfseries 88.3 \\
\bottomrule
\end{tabular*}
\endgroup
\begin{tablenotes}[flushleft]\footnotesize
\item[] \textit{Note.} Each evaluated task is excluded from temporal-model fitting and threshold or calibration selection. Only the temporal rule varies. Consistent admissions use the 60 reference-\textsc{ready} episode targets, whereas inconsistent admissions use all 720 queries. Consistency denotes agreement with the dataset reference action. Zero denotes no observed reference-inconsistent admission in this finite evaluation. Bold marks the column best. Dev denotes development-set selection. FPR denotes false-positive rate.
\end{tablenotes}
\end{threeparttable}
\end{table}

Within-task discrimination remains substantial across the six leave-one-task-out folds (equal-task-weighted area under the receiver operating characteristic curve (AUROC) 0.863, 95\% CI $[0.840,0.888]$), although task-dependent score baselines limit the transfer of the tested absolute thresholds (\Cref{fig:habit_score_origin}). A fixed 0.5 threshold admits ten reference-inconsistent candidates, and neither development-only thresholds nor Platt scaling eliminates reference-inconsistent admissions. Under the common admission policy, the within-episode positive-change rule instead retains 53 reference-consistent and zero reference-inconsistent candidates (\Cref{tab:habit_loto_temporal_transfer}). For the tested absolute rules, threshold placement limits transfer but is not established as the sole source of error.

The larger block-exchangeable analysis supports event-specific temporal pairing. Once one nonexchangeable episode is excluded, early-to-event-proximal concordance is 0.963, as against 0.926 under the task- and event-profile-stratified permutation reference ($p=0.0001$). A cross-task sensitivity analysis gives the same directional result. The evidence for pairing specificity is less conclusive in the smaller balanced 60-episode subset ($p=0.1106$). The 53 reference-consistent and zero reference-inconsistent admissions in this subset characterize the common admission policy rather than the temporal mechanism.

Positive change alone is insufficient for sequential admission. After averaging equally across tasks, the early-to-event-proximal rate of positive score change exceeds the pre-early-to-early rate by an absolute difference of 0.333 (episode-bootstrap 95\% CI $[0.233,0.433]$). The same sign rule would still admit 14 reference-inconsistent early-window queries in the offline evaluation while 53 reference-consistent queries are retained. Under this offline policy, relative change supports admission only once a temporal reference exists. Before that reference exists, the temporal condition is unmet, and the ordered policy then returns \textsc{confirm} where an otherwise selected candidate has a target mismatch and \textsc{hold} otherwise.

\begin{figure}[pos=htbp]
\centering
\includegraphics[width=0.9\textwidth]{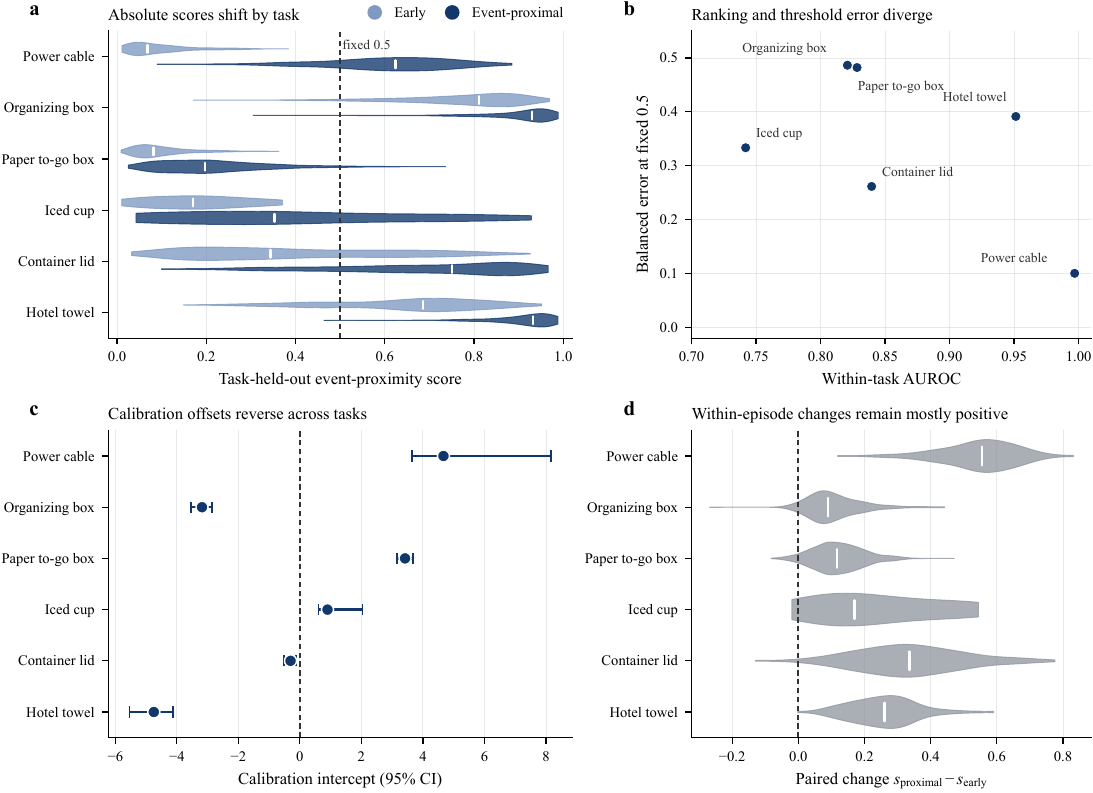}
\caption{Within-episode score changes are more stable across tasks than absolute score levels. Descriptive panels use all 1,128 events from 696 episodes. Task-specific score baselines shift while within-episode changes remain mostly positive. (a) Early and event-proximal distributions across six HABIT tasks. White ticks mark medians, and the dashed line marks 0.5. (b) Within-task area under the receiver operating characteristic curve (AUROC) versus balanced error at 0.5. (c) Calibration intercepts with episode-clustered 95\% intervals. (d) Paired early-to-proximal changes. The dashed line marks zero.}
\label{fig:habit_score_origin}
\end{figure}

Temporal tests use 10,000 fixed-point-free permutations of complete early-score profiles within each task and event-occurrence-profile stratum, with plus-one-corrected one-sided upper-tail $p$-values. Removing the sole singleton block leaves 1,126 events in 695 episodes. The statistic is the equal-task mean of episode-level strict concordance. The cross-task sensitivity analysis matches events by occurrence profile using donors from different tasks without reusing episodes. Task-dependent baseline analyses use all 1,128 pairs from 696 episodes. In the separate 60-event sequential comparison, the pre-early frame is fixed at 1.0~s before the designated early frame for each of the five cameras.

% ===========================================
% APPENDIX B.1: Sensitivity analysis and Scene15 evaluation
% ===========================================
\subsection{Sensitivity analysis and Scene15 evaluation}
\label{subsec:scene15_tmc_details}

Across all 256 Sobol settings, post-admission ncsAURC$_{[0.10,0.35]}$ remains lower for PACT than for nested Dirichlet and for provenance-discounted pooling. The corresponding ranges are 0.066--0.069, 0.107--0.223, and 0.180--0.286, and no wrong admission is observed at the operating point targeting a coverage of 0.13. Perturbation of the parameters preserves the aggregate ordering, whereas the held-condition analysis reveals reversals.

The 256-point scrambled Sobol design varies the language peak probability over $\{0.65,0.75,0.85,0.90,0.95\}$, the opinion-generation concentration over $\{50,100,200\}$, and the occlusion strength over $\{0.5,1.0,1.5\}$, while the quality and conflict multipliers are sampled within $[0.8,1.2]$. Geometry and risk perturbations preserve the original top class, and the fusion concentration remains fixed within each fold.

The source scales used with PACT have unit geometric mean and are selected by training NLL. Every fold selects approximately $(0.63,0.63,2.52)$ for language, geometry, and risk, respectively. Before admission, rescaled PACT and nested Dirichlet yield ncsAURC values of 0.632 and 0.775, respectively. After admission, the corresponding values are 0.115 and 0.148. Equal-scale PACT yields 0.629 and 0.086.

The Scene15--TMC evaluation uses 30 stratified random 80/20 train--test realizations with 897 held-out instances each. Following the preprocessing used by the TMC reference implementation~\citep{Han2021TMC}, the primary replication applies MinMax scaling independently within the training and test partitions. Mean accuracy is 66.79\% (population standard deviation (SD) 2.22 percentage points), 0.95 points below the published 67.74\%.

The provenance interventions are applied to held-out predictions from five realizations, and the learned outputs are left unchanged (4,485 predictions, 897 per realization). Original grouping, multiplicity-eight false refinement, and all-view merging yield budgets of 5.94, 19.74, and 0.316. The corresponding accuracies are 66.40\%, 61.78\%, and 57.90\%, and the corresponding ncsAURC values over $[0.01,1]$ are 0.231, 0.290, and 0.287. Scores are sorted in descending order with stable tie ordering. Cumulative risk is linearly interpolated onto 101 equally spaced coverage points and integrated by the trapezoidal rule over this interval. A 2,000-draw two-stage paired bootstrap, resampling realizations and then held-out instances within realizations, gives two-sided 95\% percentile intervals of $[12.80,14.78]$ and $[-5.94,-5.30]$ for the signed budget changes under splitting and merging relative to original grouping.

A separate preprocessing sensitivity analysis fits the scaler on the training partition and applies it unchanged to the test partition. Mean accuracy under this convention is 71.82\% across five realizations (population SD 1.88 percentage points, 4,485 held-out predictions).

\paragraph{Latency measurement.}
On the evaluated three-source benchmark running on an Intel Core Ultra 9 285K, the decision layer of PACT has a median latency of $195.8\,\mu\mathrm{s}$ and a 95th percentile of $251.7\,\mu\mathrm{s}$. The fold-specific cutoff is determined before timing and is applied unchanged. Latency covers the complete decision layer from provenance-graph construction through fusion, selection, and typed admission. It excludes perception, model inference, communication, provenance acquisition, and physical actuation. Increasing the source count from three to 64 raises the median connected-component discovery time on sparse provenance graphs with disjoint source pairs from $1.0$ to $150.8\,\mu\mathrm{s}$. The corresponding 95th percentiles are $1.1$ and $166.4\,\mu\mathrm{s}$.

% =========================================
% Data availability
% =========================================
\section*{Data availability}
Code, experimental configurations, and redistributable evaluation materials supporting the reported analyses are available at \url{https://github.com/ZekaiJ/PACT}. Third-party datasets remain subject to their original licenses.

% ===================================
% Declaration of competing interest
% ===================================
\section*{Declaration of competing interest}
Zhen Dong is a faculty member at the University of California, Santa Barbara (UCSB), and is also affiliated with NVIDIA Corporation as a researcher. The remaining authors declare that they have no known competing financial interests or personal relationships that could have appeared to influence the work reported in this paper.

% =============================================
% CRediT authorship contribution statement
% =============================================
\section*{CRediT authorship contribution statement}
\textbf{Zekai Jin:} Conceptualization, Data curation, Formal analysis, Investigation, Methodology, Software, Validation, Visualization, Writing -- original draft. \textbf{Hanrong Zhang:} Formal analysis, Writing -- review \& editing. \textbf{Yihong Tang:} Formal analysis, Methodology, Writing -- original draft. \textbf{Fei Hu:} Data curation, Writing -- original draft. \textbf{Zhen Dong:} Conceptualization, Supervision, Writing -- review \& editing. \textbf{Yi Shao:} Conceptualization, Funding acquisition, Supervision, Writing -- review \& editing.

% ===============
% Funding
% ===============
\section*{Funding}
This work was supported by the McGill Engineering Doctoral Award (MEDA). The funder had no role in study design, data generation, analysis, interpretation, manuscript preparation, or the decision to submit the article.

\renewcommand{\bibfont}{\fontsize{8pt}{8.4pt}\selectfont}
\bibliographystyle{unsrtnat}
\bibliography{cas-refs}
\end{document}